\documentclass{article}

\PassOptionsToPackage{numbers, compress}{natbib}

\usepackage[final]{neurips_2020}

\input{preamble.tex}
\newtheorem{prop}{Proposition}
\newtheorem{cond}{C}

\DeclareRobustCommand{\KL}[2]{\ensuremath{\textrm{KL}\left(#1\;\|\;#2\right)}}

\newcommand{\latent}{\mathbf{z}}
\newcommand{\dimlv}{d_z}
\newcommand{\data}{\mathbf{x}}
\newcommand{\regdata}{\mathbf{y}}
\newcommand{\varparams}{\lambda}
\newcommand{\dimvp}{d_\lambda}
\newcommand{\modparams}{\theta}

\newcommand{\step}{\varepsilon}

\newcommand{\Prb}{\mathbb{P}}
\newcommand{\Exp}{\mathbb{E}}

\newcommand{\grad}{\nabla}
\newcommand{\given}{\, | \,}
\newcommand{\idxby}{\, ; \,}
\newcommand{\eqdef}{:=}

\newcommand{\reals}{\mathbb{R}}

\newcommand{\Norm}{\mathcal{N}}

\newcommand\dif{\mathop{}\!\mathrm{d}}

\newacronym{ODE}{ODE}{ordinary differential equation}

\newacronym{SA}{SA}{stochastic approximation}
\newacronym{MSA}{MSA}{Markovian stochastic approximation}
\newacronym{SGD}{SGD}{stochastic gradient descent}

\newacronym{ML}{ML}{maximum likelihood}
\newacronym{EM}{EM}{expectation-maximization}

\newacronym{MC}{MC}{Monte Carlo}
\newacronym{MCMC}{MCMC}{Markov chain Monte Carlo}
\newacronym{IS}{IS}{importance sampling}
\newacronym{CIS}{CIS}{conditional importance sampling}
\newacronym{SMC}{SMC}{sequential Monte Carlo}
\newacronym{CSMC}{CSMC}{conditional sequential Monte Carlo}
\newacronym{SIR}{SIR}{sample importance resampling}

\newacronym{VI}{VI}{variational inference}
\newacronym{KL}{KL}{Kullback-Leibler}
\newacronym{ELBO}{ELBO}{\emph{evidence lower bound}}
\newacronym{SVI}{SVI}{stochastic variational inference}
\newacronym{CHIVI}{CHIVI}{$\chi$-divergence variational inference}
\newacronym{PSVI}{PSVI}{persistent stochastic variational inference}
\newacronym{MSC}{MSC}{Markovian score climbing}
\newacronym{EP}{EP}{expectation propagation}
\newacronym{RWS}{RWS}{reweighted wake-sleep}

\providecommand{\eg}{e.g.\@\xspace}
\providecommand{\ie}{i.e.\@\xspace}

\providecommand{\wip}{w.p.\@\xspace}
\providecommand{\cf}{cf.\@\xspace}
\providecommand{\wrt}{w.r.t.\@\xspace}

\title{Markovian Score Climbing: Variational Inference with KL(p||q)}

\author{%
  Christian A. Naesseth \\
  Columbia University, USA\\
  \texttt{christian.a.naesseth@columbia.edu} \\
  \And
  Fredrik Lindsten \\
  Link\"oping University, Sweden \\
  \texttt{fredrik.lindsten@liu.se} \\
  \And
  David Blei \\
  Columbia University, USA \\
  \texttt{david.blei@columbia.edu} \\
}

\begin{document}

\maketitle

\begin{abstract}
  Modern \gls{VI} uses stochastic gradients to avoid intractable
  expectations, enabling large-scale probabilistic inference in
  complex models.  \gls{VI} posits a family of approximating
  distributions $q$ and then finds the member of that family that is
  closest to the exact posterior $p$. Traditionally, \gls{VI}
  algorithms minimize the ``exclusive \gls{KL}'' $\KL{q}{p}$, often
  for computational convenience. Recent research, however, has also
  focused on the ``inclusive \gls{KL}'' $\KL{p}{q}$, which has good
  statistical properties that makes it more appropriate for certain
  inference problems.  This paper develops a simple algorithm for
  reliably minimizing the inclusive \gls{KL} using stochastic gradients with vanishing bias. 
This method, which we call \gls{MSC}, converges to a local optimum of the inclusive \gls{KL}. It does not suffer from the systematic errors inherent in existing methods, such as Reweighted Wake-Sleep and Neural Adaptive Sequential Monte Carlo, which lead to bias in their final
  estimates.
  We illustrate convergence on a toy
  model and demonstrate the utility of \gls{MSC} on Bayesian probit
  regression for classification
  as well as a stochastic volatility model for financial data.
\end{abstract}

\glsresetall
\section{Introduction}\label{sec:introduction}
\Gls{VI} is an optimization-based approach for approximate posterior
inference.  It posits a family of approximating distributions $q$ and
then finds the member of that family that is closest to the exact
posterior $p$. Traditionally, \gls{VI} algorithms minimize the
``exclusive \gls{KL}'' $\KL{q}{p}$ \citep{Jordan1999,blei2017variational}, which
leads to a computationally convenient optimization. For a restricted
class of models, it leads to coordinate-ascent algorithms \citep{ghahramani2001propagation}.  For a
wider class, it leads to efficient computation of unbiased gradients
for stochastic optimization \citep{Paisley2012,Salimans2013,Ranganath2014}. However, optimizing the exclusive KL
results in an approximation that underestimates the posterior
uncertainty \citep{minka2005divergence}.
To address this limitation, \gls{VI} researchers have considered
alternative divergences \citep{li2016renyi,dieng2017chi}.  One candidate is the ``inclusive \gls{KL}''
$\KL{p}{q}$
\citep{gu2015neural,Bornschein2015,finke2019importanceweighted}.  This divergence more accurately captures posterior uncertainty, but results in a challenging optimization problem.

In this paper, we develop \gls{MSC}, a simple algorithm for reliably
minimizing the inclusive \gls{KL}. Consider a valid \gls{MCMC} method
\citep{robert2013monte}, a Markov chain whose stationary distribution
is $p$.  \gls{MSC} iteratively samples the Markov chain $\latent [k]$,
and then uses those samples to follow the score function of the
variational approximation $\grad \log q(\latent[k])$ with a
Robbins-Monro step-size schedule
\citep{robbins1951stochastic}. 
Importantly, we allow the \gls{MCMC} method to depend on the current variational approximation. This enables a gradual improvement of the \gls{MCMC} as the \gls{VI} converges. We illustrate this link between the methods by using \gls{CIS} or \gls{CSMC} \citep{andrieu2010particle}.

Other \gls{VI} methods have targeted the same objective, including
\gls{RWS} \citep{Bornschein2015} and neural adaptive \gls{SMC}
\citep{gu2015neural}.  However, these methods involve biased
gradients of the inclusive \gls{KL}, which leads to bias in their
final estimates. In contrast, \gls{MSC} provides consistent gradients for essentially no added cost while providing better variational approximations. \gls{MSC} provably converges to an optimum of the inclusive \gls{KL}.

In empirical studies, we demonstrate the convergence properties
and advantages of \gls{MSC}. First, we illustrate the systematic
errors of the biased methods and how \gls{MSC} differs on a toy
skew-normal model. Then we compare \gls{MSC} with \gls{EP} and
\gls{IS}-based optimization
\citep{Bornschein2015,finke2019importanceweighted} on a Bayesian
probit classification example with benchmark data. Finally, we apply
\gls{MSC} and \gls{SMC}-based optimization \citep{gu2015neural} to fit
a stochastic volatility model on exchange rate data.

\paragraph{Contributions.} The contributions of this paper are
\begin{enumerate*}[label=(\roman*)]
\item developing \acrlong{MSC}, a simple algorithm that provably minimizes $\KL{p}{q}$; 
\item studying systematic errors in existing methods that lead to
  bias in their variational approximation; and
\item empirical studies that confirm convergence and illustrates the utility of \gls{MSC}.
\end{enumerate*}

\paragraph{Related Work.}  Much recent effort in \gls{VI} has focused
on optimizing cost functions that are not the exclusive \gls{KL}
divergence. For example R\'{e}nyi divergences and $\chi$ divergence are studied in  \citep{li2016renyi,dieng2017chi}. The most
similar to our work are the methods in
\citep{Bornschein2015,gu2015neural,finke2019importanceweighted}, using
\gls{IS} or \gls{SMC} to optimize the inclusive \gls{KL}
divergence. The \gls{RWS} algorithm \citep{Bornschein2015} uses
\gls{IS} both to optimize model parameters and the variational
approximation. Neural adaptive \gls{SMC} \citep{gu2015neural}
jointly learn an approximation to the posterior and optimize the
marginal likelihood of time series with gradients estimated by
\gls{SMC}. In \citep{finke2019importanceweighted} connections
between importance weighted autoencoders \citep{Burda2016}, adaptive
\gls{IS} and methods like the \gls{RWS} are drawn. These three works all rely on
\gls{IS} or \gls{SMC} to estimate expectations with respect to the
posterior. This introduces a systematic bias in the gradients that
leads to a solution which is not a local optimum to the inclusive
\gls{KL} divergence. In \citep{paige2016inference} inference networks are learnt for data simulated from the model rather than observed data.

Another line of work studies the combination of \gls{VI} with \gls{MC}
methods. \citet{salimans2015markov} take inspiration from the
\gls{MCMC} literature to define their variational approximation.  The
method in \citep{hoffman2019neutralizing} uses the variational
approximation to improve Hamiltonian \gls{MC}.  Variational \gls{SMC}
 \citep{naesseth18a,anh2018autoencoding,maddison2017}
uses the \gls{SMC} sample process itself to define an approximation to
the posterior. Follow up work
\citep{lawson2018twisted,moretti2019particle} improve on variational
\gls{SMC} in various ways by using \emph{twisting}
\citep{guarniero2017iterated,heng2017controlled,lindsten2018}. 
Another approach takes a \gls{MC} estimator of the marginal
likelihood and turn it into a posterior approximation \citep{domke2019divide}.  The method in
\citep{habib2018auxiliary} uses auxiliary variables to define a more
flexible approximation to the posterior, then subsequently at test
time apply \gls{MCMC}.  These methods all optimize a variational
approximation based on \gls{MC} methods to minimize the exclusive
\gls{KL} divergence. On the contrary, the method proposed in this
paper minimizes the inclusive \gls{KL} divergence.
The method in \citep{hoffman2017learning} optimizes an initial
approximation to the posterior in exclusive \gls{KL}, then refines
this with a few iterations of \gls{MCMC} to estimate gradients with
respect to the model parameters.  Defining the
variational approximation as an initial distribution to which a few
steps of \gls{MCMC} is applied, and then optimize a new contrastive
divergence is done in \citep{ruiz2019a}. This divergence is different from the inclusive \gls{KL} and \gls{MCMC} is used as a part of the variational approximation
rather than gradient estimation.
Another line of work studies combinations of \gls{MC} and \gls{VI} using amortization \citep{li2017approximate,wang2018meta,wu2020}.

Using \gls{MC} together with stochastic optimization, for \eg maximum likelihood estimation of latent variable models, is studied in \citep{Gu7270,KuhnL:2004,AndrieuV:2014,dieng2019reweighted}. In contrast the proposed method uses it for \gls{VI}.
Viewing \gls{MSC} as a way to learn a proposal distribution for \gls{IS} means it is related to the class of adaptive \gls{IS} algorithms \citep{bugallo2017adaptive,cappe2004population,douc2007convergence,
cappe2008adaptive}. We compare to \gls{IS}/\gls{SMC}-based optimization, as outlined in the background, in the experimental studies which can be considered to be special cases of adaptive \gls{IS}/\gls{SMC}.

Concurrent and independent work using \gls{MCMC} to optimize the inclusive \gls{KL} was studied in \citep{ou2020joint}. The difference with our work lies in the Markov kernels used, our focus on continuous latent variables, and our study of the impact of large-scale exchangeable data.


\section{Background}
\label{sec:background}
Let $p(\latent,\data)$ be a probabilistic model for the latent
(unobserved) variables $\latent$ and data $\data$. In Bayesian
inference the main concern is computing the posterior distribution
$p(\latent\given\data)$, the conditional distribution of the latent
variables given the observed data. The posterior is $p(\latent\given\data) = \nicefrac{p(\latent,\data)}{p(\data)}$. The normalization constant is the marginal likelihood $p(\data)$,
computed by integrating (or summing) the joint model
$p(\latent,\data)$ over all values of $\latent$. For most models of
interest, however, exactly computing the posterior is intractable, and
we must resort to a numerical approximation.

\subsection{Variational Inference with KL(p||q)}
One approach to approximating the posterior is with \gls{VI}.
This turns the intractable problem of computing the posterior into
an optimization problem that can be solved numerically.  The idea is
to first posit a variational family of approximating distributions
$q(\latent\idxby \varparams)$, parametrized by $\varparams$. Then
minimize a metric or divergence so that the variational approximation
is close to the posterior,
$q(\latent\idxby \varparams) \approx p(\latent\given\data)$.

The most common \gls{VI} objective is to minimize the exclusive
\gls{KL}, $\KL{q}{p}$.  This objective is an expectation with respect
to the approximating distribution $q$ that is convenient to
optimize. But this convenience comes at a cost---the $q$ optimized to
minimize $\KL{q}{p}$ will underestimate the variance of the posterior 
\citep{dieng2017chi,blei2017variational,turner2011two}.

One way to mitigate this issue is to instead optimize the inclusive
\gls{KL},
\begin{align}
  \KL{p(\latent\given\data)}{q(\latent\idxby \varparams)} \eqdef \Exp_{ p(\latent\given\data)}\left[ \log p(\latent\given\data)-\log q(\latent\idxby \varparams) \right].
	\label{eq:divoptproblem}
\end{align}
This objective, though more difficult to work with, does not lead to
underdispersed approximations. For too simplistic $q$ it might lead to approximations that stretch to cover $p$ putting mass even where $p$ is small, thus leading to poor predictive distributions. However, in the context of \gls{VI} inclusive \gls{KL} has
motivated among others neural adaptive \gls{SMC}~\citep{gu2015neural},
\gls{RWS}~\citep{Bornschein2015}, and \gls{EP}
\citep{minka2001expectation}. This paper develops \gls{MSC}, a new
algorithm to minimize the inclusive \gls{KL} divergence.

Minimizing \cref{eq:divoptproblem} is equivalent to minimizing the cross entropy $L_{\textrm{KL}}(\varparams)$,
\begin{align}
  \min_\varparams \,
  L_{\textrm{KL}}(\varparams) &\eqdef \min_\varparams \,
                                \Exp_{ p(\latent\given\data)}\left[- 
                                \log  q(\latent\idxby \varparams) \right] \label{eq:kloptproblem}.
\end{align}
The gradient \wrt the variational parameters is
\begin{align}
  g_{\textrm{KL}}(\varparams) &\eqdef
                                \grad L_{\textrm{KL}}(\varparams) = \Exp_{ p(\latent\given\data)}\left[ - s(\latent\idxby \varparams) \right],
                                \label{eq:fixedpointeq}
\end{align}
where we define $s(\latent\idxby \varparams)$ to be the \emph{score function},
\begin{align}
  s(\latent\idxby \varparams) \eqdef \grad_\varparams \log q(\latent\idxby\varparams).
  \label{eq:score}
\end{align}
Because the cross entropy is an expectation with respect to the
(intractable) posterior, computing its gradient pointwise is
intractable. 
Recent algorithms for solving \cref{eq:kloptproblem} focus on \emph{stochastic gradient descent}~\citep{Bornschein2015,gu2015neural,finke2019importanceweighted}.

\subsection{Stochastic Gradient Descent with IS}\label{sec:bkg:biased}

We use \gls{SGD} in \gls{VI} when the gradients of the objective are
intractable. The \gls{SGD} updates 
\begin{align}
  \varparams_k &= \varparams_{k-1} - \step_k\, \widehat{g}_{\textrm{KL}}(\varparams_{k-1}),
                 \label{eq:sa:algorithm}
\end{align}
converges to a local optimum of \cref{eq:kloptproblem} if the gradient estimate $\widehat{g}_{\textrm{KL}}$ is
unbiased, $\Exp\left[\widehat{g}_{\textrm{KL}}(\varparams)\right] =
g_{\textrm{KL}}(\varparams)$, and the step sizes satisfy
$\sum_k \step_k^2 < \infty$, $\sum_k \step_k = \infty$
\citep{robbins1951stochastic,kushner2003stochastic}. 

When the objective is the exclusive $\KL{q}{p}$, we can
use score-function gradient estimators
\citep{Paisley2012,Salimans2013,Ranganath2014}, reparameterization
gradient estimators \citep{Rezende2014,Kingma2014}, or combinations of
the two \citep{Ruiz2016,Naesseth2017}. These methods provide unbiased
stochastic gradients that can help find a local optimum of the
exclusive \gls{KL}. 

However, we consider minimizing the inclusive $\KL{p}{q}$ \cref{eq:divoptproblem}, for which gradient estimation is difficult. It requires
an expectation with respect to the posterior $p$.
One strategy is to use \gls{IS} \citep{robert2013monte} to rewrite the
gradient as an expectation with respect to $q$. Specifically, the
gradient of the inclusive \gls{KL} is proportional to
\begin{align}
  \grad_\varparams L_{\textrm{KL}}(\varparams) &\propto -\Exp_{ q(\latent\idxby \varparams) } \left[\frac{p(\latent,\data)}{q(\latent\idxby \varparams)} s(\latent\idxby \varparams) \right],
\end{align}
where the constant of proportionality $\nicefrac{1}{p(\data)}$ is
independent of the variational parameters and will not affect the
solution of the corresponding fixed point equation. This gradient is
unbiased, but estimating it using standard \gls{MC} methods can lead
to high variance and poor convergence.

Another option \citep{gu2015neural,Bornschein2015} is the
self-normalized \gls{IS} (or corresponding \gls{SMC}) estimate 
\begin{align}
  \grad_\varparams L_{\textrm{KL}}(\varparams) &\approx -\sum_{s=1}^S \frac{w_s}{\sum_{r=1}^S w_r} s (\latent^s \idxby \varparams) , \label{eq:sgd:inclkl}
\end{align}
where
$w_s = \nicefrac{p(\latent^s,\data)}{q(\latent^s\idxby \varparams)}$,
$\latent^s \sim q(\latent^s \idxby \varparams)$, and
$s(\latent\idxby\varparams)=\grad_\varparams \log
q(\latent\idxby\varparams)$. However, \cref{eq:sgd:inclkl} is
not unbiased. The estimator suffers from systematic error and,
consequently, the fitted variational parameters are no longer optimal
with respect to the original minimization problem in
\cref{eq:kloptproblem}. (See
\citep{naesseth2019elements,mcbook,robert2013monte} for details
about \gls{IS} and \gls{SMC} methods)
\gls{MSC} addresses this shortcoming, introducing an algorithm that provably converges to a solution of \cref{eq:kloptproblem}. In the remainder of the paper \gls{IS} refers to self-normalized \gls{IS}.



\section{Markovian Score Climbing}\label{sec:method}
The key idea in \gls{MSC} is to use \gls{MCMC} methods to estimate the intractable gradient. Under suitable conditions on the algorithm, \gls{MSC} is guaranteed to converge to a local optimum of $\KL{p}{q}$. 

First, we discuss generic \gls{MCMC} methods to estimate gradients in a \gls{SGD} algorithm. Importantly, the \gls{MCMC} method can depend on the current \gls{VI} approximation which provides a tight link between \gls{MCMC} and \gls{VI}. Next we exemplify this connection by introducing \gls{CIS}, an example Markov kernel that is a simple modification of \gls{IS}, where the \gls{VI} approximation is used as a proposal. 
The extra computational cost is negligible compared to the biased approaches discussed in \cref{sec:bkg:biased}, \gls{CIS} only generates a single extra categorical random variable per iteration. The corresponding extension to \gls{SMC}, \ie the \gls{CSMC} kernel, is discussed in the supplement. Next, we discuss learning model parameters. Then,  we show that the resulting \gls{MSC} algorithm is exact in the sense that it converges asymptotically to a local optima of the inclusive \gls{KL} divergence. Finally, we discuss large-scale data.

%
%
\subsection{Stochastic Gradient Descent using MCMC}
When using gradient descent to optimize the inclusive \gls{KL} we must compute an expectation of the score function $s(\latent\idxby\varparams)$ \cref{eq:score} with respect to the true posterior. To avoid this intractable expectation we propose to use stochastic gradients estimated using samples generated from a \gls{MCMC} algorithm, with the posterior as its stationary distribution. The key step to ensure convergence, without having to run an infinite inner loop of \gls{MCMC} updates, is to \emph{not} re-initialize the Markov chain at each step $k$. Instead, the sample $\latent[k-1]$ used to estimate the gradient at step $k-1$ is passed to a Markov kernel $\latent[k] \sim M(\cdot\given \latent[k-1])$, with the posterior as its stationary distribution, to get an updated $\latent[k]$ that is then used to estimate the current gradient, \ie the score $s(\latent[k]\idxby\varparams)$. This leads to a Markovian stochastic approximation algorithm \citep{Gu7270}, where the noise in the gradient estimate is Markovian. Because we are moving in an ascent direction of the score function at each iteration and using \gls{MCMC}, we refer to the method developed in this paper as \acrlong{MSC}.

 It is not a requirement that the Markov kernel $M$ is independent of the variational parameters $\varparams$. In fact it is key for best performance of \gls{MSC} that we use the variational approximation to define the Markov chain. We summarize \gls{MSC} in \cref{alg:msa}. 
\begin{algorithm}[th]
\DontPrintSemicolon
\SetKwInOut{Input}{Input}\SetKwInOut{Output}{Output}
\Input{Markov kernel $M(\latent'\given \latent\idxby\varparams)$ with stationary distribution $p(\latent\given\data)$, variational family $q(\latent\idxby\varparams)$, initial $\varparams_0$, initial $\latent[0]$, step size sequence $\step_k$, and number of iterations $K$.}
 \Output{$\varparams_K \approx \varparams^\star$.}
 \BlankLine
 \For{$\displaystyle k=1,\ldots,K$}{
    Sample $\latent[k] \sim M(\cdot \given \latent[k-1]\idxby\varparams_{k-1})$\;
    
    Compute $s(\latent[k]\idxby\varparams_{k-1})= \grad_\varparams\log q(\latent[k]\idxby\varparams_{k-1})$\;
    
    Set $\varparams_k = \varparams_{k-1}+\step_k s(\latent[k]\idxby\varparams_{k-1})$\;
}	
 \caption{Markovian Score Climbing}\label{alg:msa}
\end{algorithm}

Next, we discuss \gls{CIS} \citep{andrieu2010particle,naesseth2019elements}, an example Markov kernel with adaptation that is a simple modifications of its namesake \gls{IS}. The corresponding extension to \gls{SMC}, the \gls{CSMC} kernel \citep{andrieu2010particle,naesseth2019elements}, is discussed in the supplement. Using these Markov kernels to estimate gradients, rather than \gls{IS} and \gls{SMC} \citep{gu2015neural,Bornschein2015}, lead to algorithms that are simple modifications of their non-conditional counterparts but provably converge to a local optimum of the inclusive \gls{KL} divergence.

\subsection{Conditional Importance Sampling}
\gls{CIS} is an \gls{IS}-based Markov kernel with $p(\latent\given\data)$ as its stationary distribution \citep{andrieu2010particle,naesseth2019elements}. It modifies the classical \gls{IS} algorithm by retaining one of the samples from the previous iteration, the so-called \emph{conditional sample}. Each iteration consists of three steps: generate new samples from a proposal, compute weights, and then update the conditional sample for the next iteration. We explain in detail below.

First, set the first proposed sample to be equal to the conditional sample from the previous iteration, \ie $\latent^1 = \latent[k-1]$, and propose the remaining $S-1$ samples from a proposal distribution $\latent^i \sim q(\latent\idxby\varparams), \quad i=2,\ldots,S.$
The proposal does not necessarily need to be equal to the variational approximation, a common option is to use the model prior $p(\latent)$. However, we will in the remainder of this paper assume that the variational approximation $q(\latent\idxby\varparams)$ is used as the proposal. This provides a link between the \gls{MCMC} proposal and the current \gls{VI} approximation.
Then, compute the importance weights for all $S$ samples, including the conditional sample. The importance weights for $i=1,\ldots,S$ are
$w^i = \nicefrac{p(\latent^i,\data)}{q(\latent^i\idxby\varparams)}$, $\bar{w}^i = \nicefrac{w^i}{\sum_{j=1}^S w^j}$.
Finally, generate an updated conditional sample by picking one of the proposed values with probability proportional to its (normalized) weight, \ie, $\latent[k] = \latent^J$,
where $J$ is a discrete random variable with probability $\Prb(J=j) = \bar{w}^j$.

Iteratively repeating this procedure constructs a Markov chain with the posterior $p(\latent\given\data)$ as its stationary distribution \citep{andrieu2010particle,naesseth2019elements}. With this it is possible to attain an estimate of the (negative) gradient \wrt the variational parameters of \cref{eq:kloptproblem}:
\begin{align}
	s(\latent[k]\idxby\varparams) &= \grad_{\varparams} \log q(\latent[k]\idxby\varparams),
	\label{eq:kloptproblem:cisgradient}
\end{align}
where $\latent[k]$ is the conditional sample retained at each iteration of the \gls{CIS} algorithm. Another option is to make use of all samples at each iteration, \ie the Rao-Blackwellized estimate, $\widehat{g}_{\textrm{KL}}(\varparams) = \sum_{i=1}^S \bar{w}^i  s(\latent^i\idxby\varparams).$
We summarize one full iteration of the \gls{CIS} in \cref{alg:cis}. 
\begin{algorithm}
\DontPrintSemicolon
\SetKwInOut{Input}{Input}\SetKwInOut{Output}{Output}
 \Input{Model $p(\latent,\data)$, proposal $q(\latent\idxby\varparams)$, conditional sample $\latent[k-1]$, and total number of internal samples $S$.}
 \Output{$\latent[k] \sim M(\cdot\given\latent[k-1]\idxby\varparams)$, updated conditional sample.}
 \BlankLine
 Set $\displaystyle \latent^1 = \latent[k-1]$\;
 
 Sample $\displaystyle \latent^i \sim q(\latent\idxby\varparams)$ for $\displaystyle i=2,\ldots,S$\;
 
 Compute $\displaystyle w^i = \nicefrac{p(\latent^i,\data)}{q(\latent^i\idxby\varparams)}$, $\displaystyle \bar{w}^i = \nicefrac{w^i}{\sum_{j=1}^S w^j}$ for $\displaystyle i=1,\ldots,S$ \;
 
 Sample $\displaystyle J$ with probability $\displaystyle \Prb(J=j) \propto  \bar{w}^j$\;
 
 Set $\displaystyle \latent[k] = \latent^J$\;
 \caption{Conditional Importance Sampling}\label{alg:cis}
\end{algorithm}


\subsection{Model Parameters}\label{sec:modparams}
If the probabilistic model has unknown parameters $\modparams$ one solution is to assign them a prior distribution, include them in the latent variable $\latent$, and apply the method outlined above to approximate the posterior. However, an alternative solution is to use the \gls{ML} principle and optimize the marginal likelihood, $p(\data\idxby\modparams)$, jointly with the approximate posterior, $q(\latent\idxby\varparams)$. We propose to use \acrlong{MSC} based on the Fisher identity of the gradient
\begin{align}
	g_{\textrm{ML}}(\modparams) &= \grad_\modparams \log p(\data\idxby\modparams) = \grad_\modparams \log \int p(\latent, \data\idxby\modparams)\dif\latent= \Exp_{p_\modparams(\latent\given\data)}\left[\grad_\modparams \log p(\latent, \data\idxby\modparams)\right].
\end{align}
With a Markov kernel $M(\latent[k]\given \latent[k-1]\idxby \modparams,\varparams)$, with the posterior distribution $p(\latent\given\data\idxby\modparams)$ as its stationary distribution, the approximate gradient is
	$\widehat{g}_{\textrm{ML}}(\modparams) = \grad_\modparams \log p(\latent[k],\data\idxby\modparams)$.

The \gls{MSC} algorithm for maximization of the log-marginal likelihood, with respect to $\modparams$, and minimization of the inclusive \gls{KL} divergence, with respect to $\varparams$, is summarized in \cref{alg:msa:modparams}.
Using \gls{MSC} \emph{only} for \gls{ML} estimation of $\theta$, with a fixed Markov kenel $M$ and \emph{without} the \gls{VI} steps on lines 13 and 15, is equivalent to the \gls{MCMC}  \gls{ML} method in \citep{Gu7270}.
\begin{algorithm}
\DontPrintSemicolon
\SetKwInOut{Input}{Input}\SetKwInOut{Output}{Output}
 \Input{Markov kernel $M(\latent'\given \latent\idxby \modparams,\varparams)$ with stationary distribution $p(\latent\given\data\idxby\modparams)$, variational family $q(\latent\idxby\varparams)$, initial $\varparams_0, \latent[0], \modparams_0$, step size sequences $\step_k, \epsilon_k$, and iterations $K$.}
 \Output{$\varparams_K \approx \varparams^\star$, $\modparams_K \approx \modparams^\star$.}
 \BlankLine
 \For{$\displaystyle k=1,\ldots,K$}{
    Sample $\latent[k] \sim M(\cdot\given \latent[k-1]\idxby\modparams_{k-1},\varparams_{k-1})$\;
    
    Compute $s(\latent[k]\idxby\varparams_{k-1}) = \grad_\varparams\log q(\latent[k]\idxby\varparams_{k-1})$\;
    
    Compute $\widehat{g}_{\textrm{ML}}(\modparams_{k-1}) = \grad_\modparams \log p(\latent[k],\data\idxby\modparams_{k-1})$\;
    
    Set $\varparams_k = \varparams_{k-1}+\step_k s(\latent[k]\idxby\varparams_{k-1})$\;
    
    Set $\modparams_k = \modparams_{k-1}+\epsilon_k \widehat{g}_{\textrm{ML}}(\modparams_{k-1})$\;
}
 \caption{Markovian Score Climbing with ML}\label{alg:msa:modparams}
\end{algorithm}

\subsection{The Convergence of MSC}\label{seq:theory}
One of the main benefits of \gls{MSC} is that it is possible, under certain regularity conditions, to ensure that the variational parameter estimate $\varparams_K$ as provided by \cref{alg:msa} converges to a local optima of the inclusive \gls{KL} divergence as the number of iterations $K$ tend to infinity. 
We formalize the convergence result in \cref{prop:msa}. The result is an application of \citep[Theorem 1]{Gu7270} and based on  \citep[Theorem 3.17, page 304]{benveniste2012adaptive}. The proof is found in the Supplement.
\begin{prop}
Assume that C$1$--C$6$, detailed in the supplement, hold. If $\lambda_k$ for $k\geq 1$ defined by \cref{alg:msa} is a bounded sequence and almost surely visits a compact subset of the domain of attraction of $\varparams^\star$ infinitely often, then
\begin{align*}
\varparams_k \to \varparams^\star, \quad \textrm{almost surely.}
\end{align*}
\label{prop:msa}
\end{prop}

\subsection{MSC on Large-Scale Data}\label{sec:bigdata}
If the dataset $\data = (x_1,,\ldots,x_n)$ is large it might be impractical to evaluate the full likelihood at each step and it would be preferable to consider only a subset of the data at each iteration. Variational inference based on the exclusive \gls{KL}, $\KL{q}{p}$, is scalable in the sense that it works by subsampling datasets both for \emph{exchangeable data}, $p(\data) = \mathbb{E}_{p(\latent)}\left[\prod_{i=1}^n p(x_i \given \latent)\right]$, as well as for \emph{independent and identically distributed data} (iid), $p(\data) = \prod_{i=1}^n p(x_i) = \prod_{i=1}^n \mathbb{E}_{p(z_i)}\left[p(x_i \given z_i)\right]$ where $\latent=(z_1,\ldots,z_n)$. For the exclusive \gls{KL} divergence subsampling is straightforward; the likelihood enters as a sum of the individual log-likelihood terms for all datapoints whether the data is iid or exchangeable, and a simple unbiased estimate can be constructed by sampling one (or a few) datapoints to evaluate at each iteration. However, for the inclusive \gls{KL} divergence the large-scale data implications for the two settings are less clear and we discuss each below.

Often in the literature \citep{Bornschein2015,Burda2016,dieng2019reweighted,ou2020joint} applications assumes the data is generated \emph{iid} and achieve scalability through use of subsampling and amortization. In fact, \gls{MSC} can potentially scale just as well as other algorithms to large datasets when data is assumed iid $x_i \sim p(x), ~i=1,\ldots,n$. Instead of minimizing $\KL{p(z\given x)}{q(z\idxby \varparams)}$ wrt $\lambda$ for each $x=x_i$, we consider minimizing $\KL{p(x) p(z|x)}{p(x) q(z\given\varparams_\eta(x))}$ wrt $\eta$ where $\lambda_\eta(x)$ is an inference network (amortization). If $q(z\given\varparams_\eta(x))$ is flexible enough the posterior $p(z\given x)$ is the optimal solution to this minimization problem. Stochastic gradient descent can be performed by noting that
\begin{align*}
&\nabla_\eta \KL{p(x) p(z|x)}{p(x) q(z\given \varparams_\eta(x))} = 0+\mathbb{E}_{p(x) p(z|x)}\left[-\nabla_\eta \log  q(z|\lambda_\eta(x)\right] \\
&\approx \frac{1}{n} \sum_{i=1}^n\mathbb{E}_{p(z|x_i)}\left[-\nabla_\eta \log q(z|\lambda_\eta(x_i))\right],
\end{align*}
where the approximation is directly amenable to data subsampling. We leave the formal study of this approach for future work.

For \emph{exchangeable data} the likelihood enters as a product and subsampling is difficult in general.
Standard \gls{MCMC} kernels require evaluation of the complete likelihood at each iteration, which means that the method proposed in this paper likewise must evaluate all the data points at each iteration of \cref{alg:msa}. An option is to follow \citep{li2016renyi,dieng2017chi} using \emph{subset average likelihoods}. In \cref{app:like} we prove that this approach leads to systematic errors that are difficult to quantify. It does not minimize the inclusive \gls{KL} from $p$ to $q$, rather it minimizes the \gls{KL} divergence from a \emph{perturbed} posterior $\widetilde{p}$ to $q$.
A potential remedy to this issue, that we leave for future work, is to consider approximate \gls{MCMC} (with theoretical guarantees) reviewed in \eg \citep{bardenet2017,angelino2016patterns}.

\section{Empirical Evaluation}\label{sec:experiments}
We illustrate convergence on a toy model and demonstrate the utility of \gls{MSC} on Bayesian probit regression for classification as well as a stochastic volatility model for financial data. 
The studies show that \gls{MSC}
\begin{enumerate*}[label=(\roman*)]
\item converges to the true solution whereas the biased methods do not; 
\item achieves similar predictive performance as \gls{EP} and \gls{IS} on regression while being more robust to the choice of sample size $S$; and
\item learns superior or as good stochastic volatility models as \gls{SMC}.
\end{enumerate*}
Code is available at \url{github.com/blei-lab/markovian-score-climbing}.

\subsection{Skew Normal Distribution}
We illustrate the impact of the biased gradients discussed in \cref{sec:bkg:biased} on a toy example. Let $p(\latent\given\data)$ be a scalar skew normal distribution with location, scale and shape parameters $(\xi, \omega, \alpha) = (0.5, 2, 5)$. We let the variational approximation be a family of normal distributions $q(\latent\idxby \varparams) = \Norm(\latent\idxby \mu, \sigma^2)$. For this choice of posterior and approximating family it is possible to compute the analytical solution for the inclusive \gls{KL} divergence; it corresponds to matching the moments of the variational approximation and the posterior distribution. 
In \cref{fig:biased} we show the results of \gls{SGD} when using the biased gradients from \cref{eq:sgd:inclkl}, \ie using self-normalized \gls{IS} to estimate the gradients, and \gls{MSC} (this paper) as described in \cref{sec:method}. We set the number of samples to $S=2$. We can see how the biased gradient leads to systematic errors when estimating the variational parameters, whereas \gls{MSC} obtains the true solution. Increasing the number of samples for the estimator in \cref{eq:sgd:inclkl} will lower the bias, and in the limit of infinite samples $S$ it is exact. However, for non-toy problems it is likely very difficult to know what is a sufficient number of samples to get an "acceptable bias" in the \gls{VI} solution. \gls{MSC}, on the other hand, provides consistent estimates of the variational parameters even with small number of samples.
\begin{figure}[ht]
	\centering
	\begin{subfigure}[b]{0.35\textwidth}
        \includegraphics[width=\textwidth]{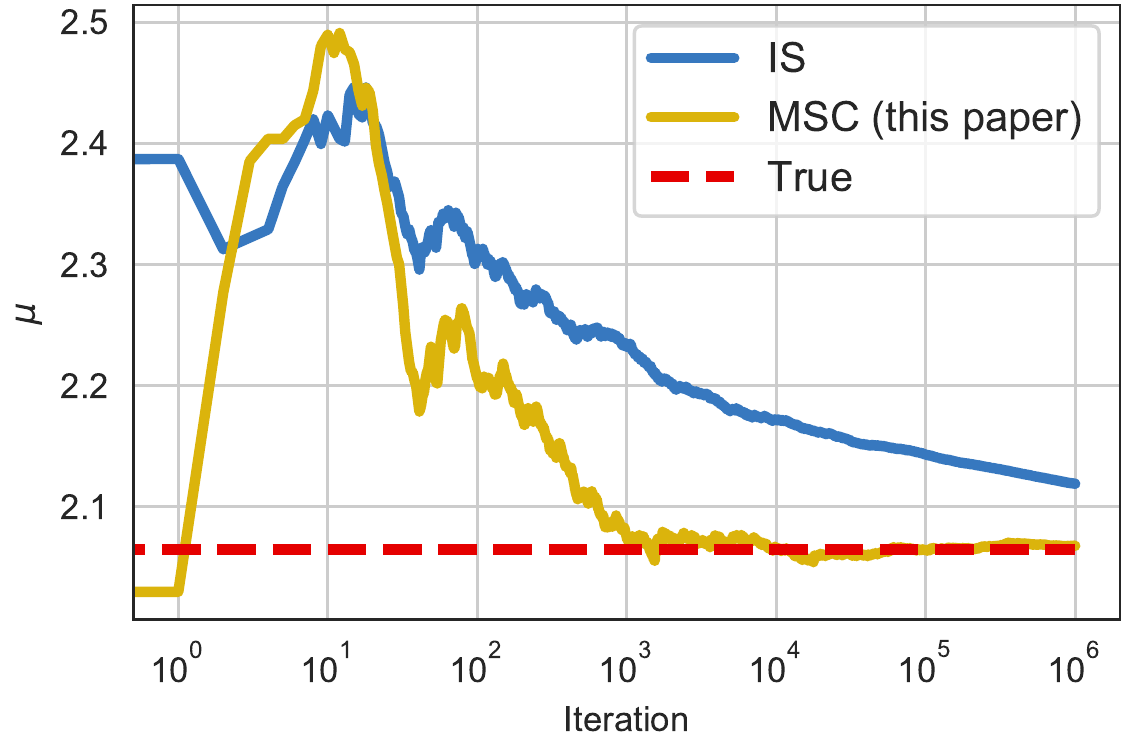}
    \end{subfigure}
    ~ 
    \begin{subfigure}[b]{0.35\textwidth}
        \includegraphics[width=\textwidth]{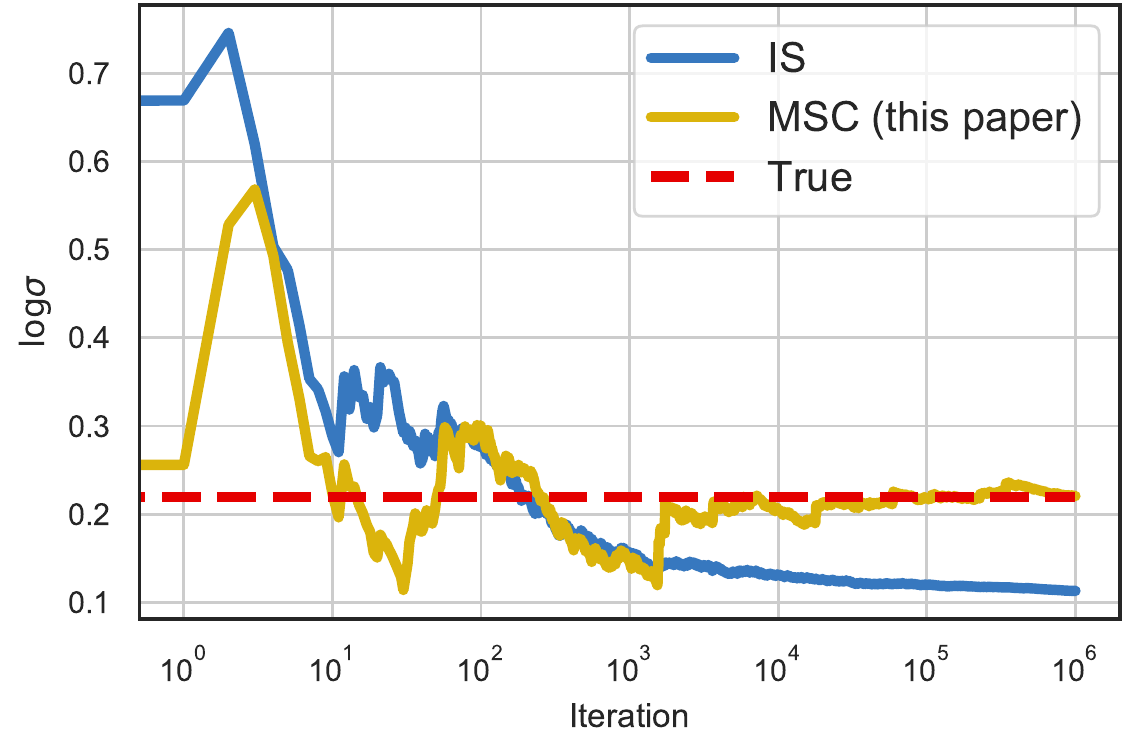}
    \end{subfigure}
	\caption{\Gls{MSC} converges to the true solution, while the biased \gls{IS} approach does not. Example of learnt variational parameters for \gls{IS}- and \gls{MSC}-based gradients of the inclusive \gls{KL}, as well as true parameters. Gaussian approximation to a skew normal distribution. Iterations in log-scale.}\label{fig:biased}
\end{figure}
Note that the biased \gls{IS}-gradients results in an \emph{underestimation} of the variance. One of the main motivations for using inclusive \gls{KL} as optimization objective is to avoid such underestimation of uncertainty. This example shows that when the inclusive \gls{KL} is optimized with \emph{biased gradients} the solution can no longer be trusted in this respect.
The gradients for R\'{e}nyi- and $\chi$ divergences used in \eg \citep{li2016renyi,dieng2017chi} suffer from a similar bias. The supplement provides a $\chi$ divergence analogue to \cref{fig:biased}.
\begin{figure*}[t]
	\centering
	\begin{subfigure}[b]{0.3\textwidth}
        \includegraphics[width=\textwidth]{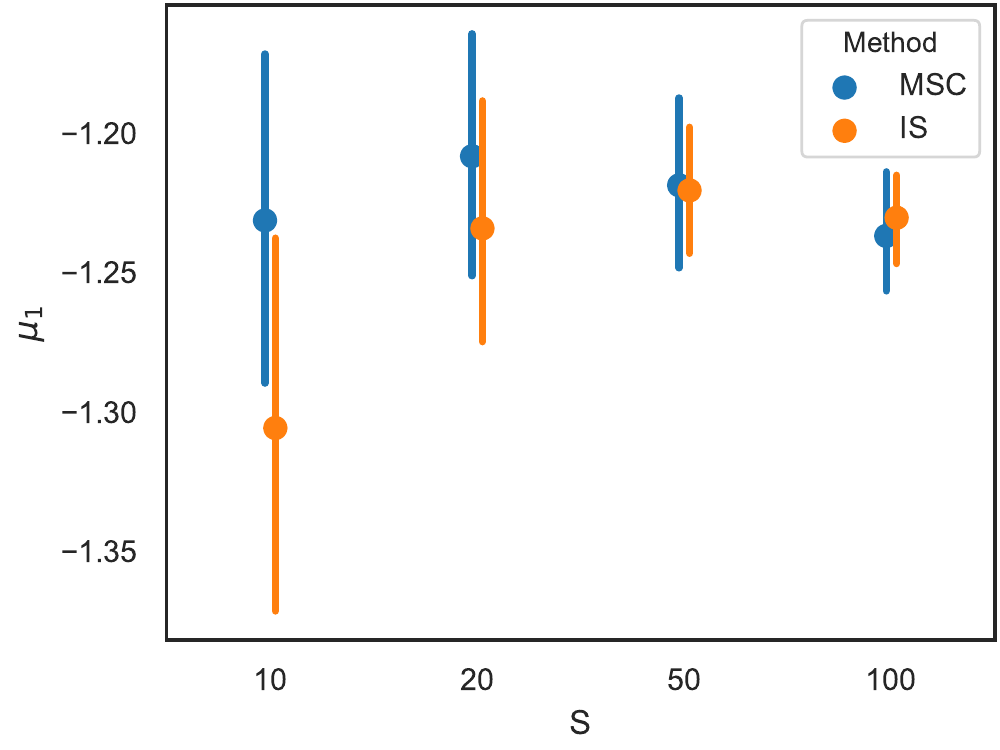}
        \caption{Heart $\mu_1^\star$}
    \end{subfigure}
    ~ 
    \begin{subfigure}[b]{0.3\textwidth}
        \includegraphics[width=\textwidth]{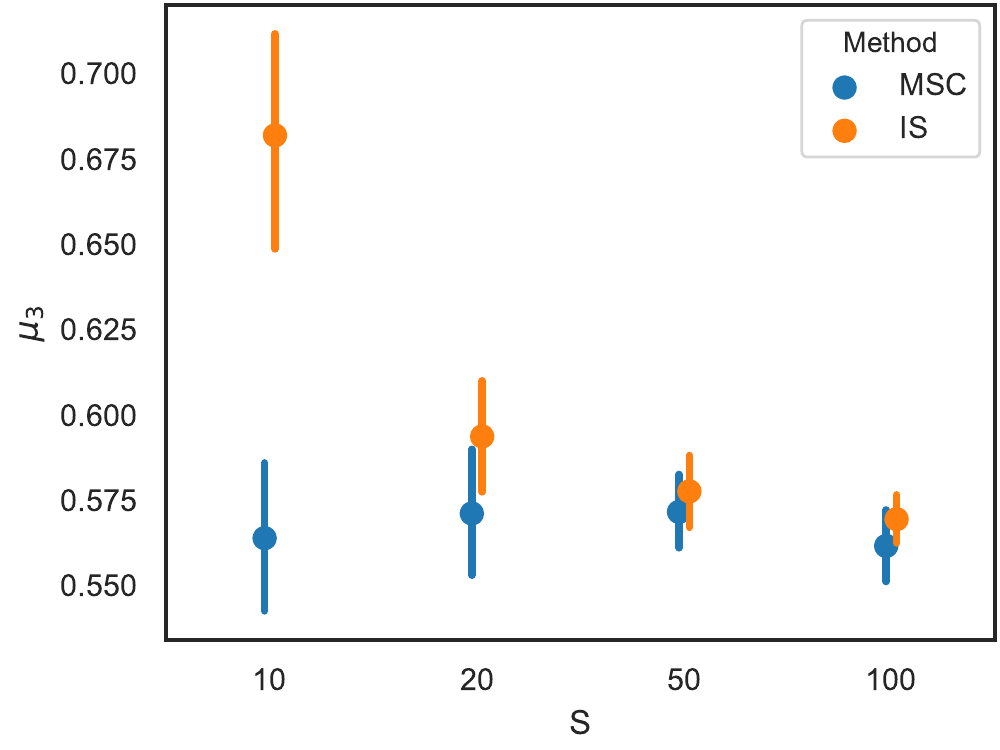}
        \caption{Heart $\mu_3^\star$}
    \end{subfigure}
    ~ 
    \begin{subfigure}[b]{0.3\textwidth}
        \includegraphics[width=\textwidth]{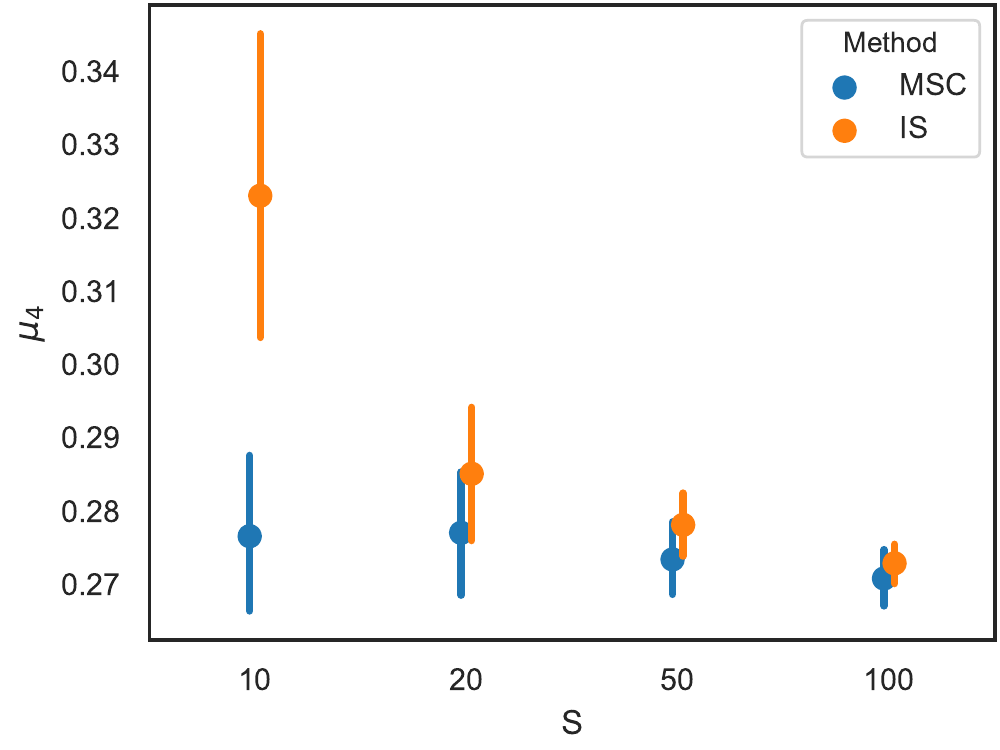}
        \caption{Heart $\mu_4^\star$}
    \end{subfigure}
    
    	\begin{subfigure}[b]{0.3\textwidth}
        \includegraphics[width=\textwidth]{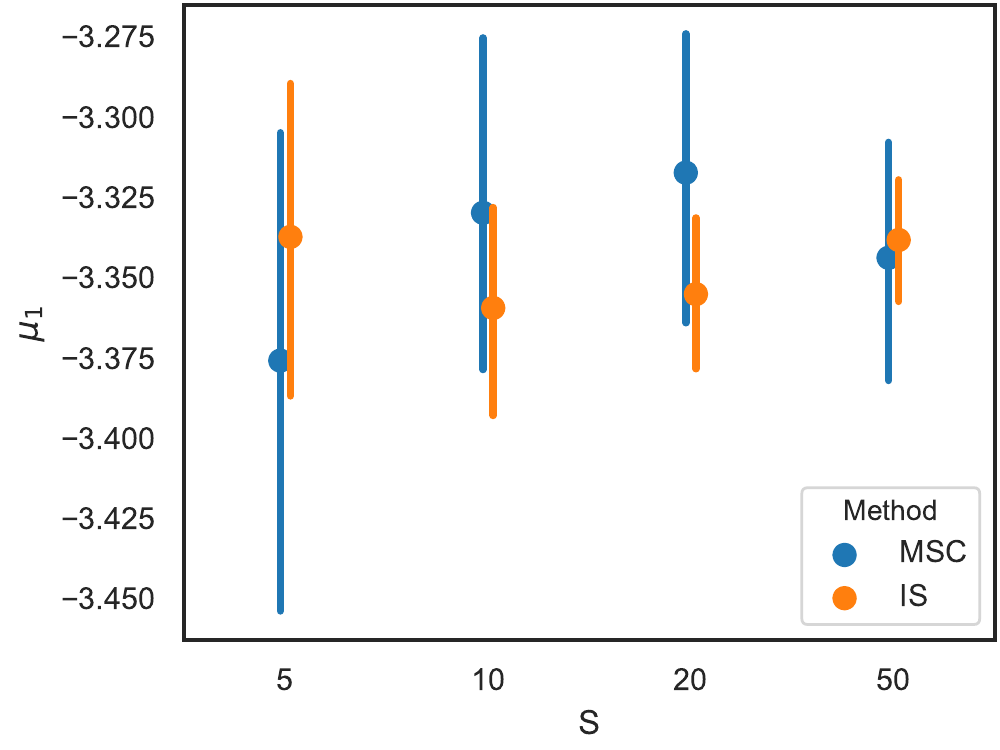}
        \caption{Ionos $\mu_1^\star$}
    \end{subfigure}
    ~ 
    \begin{subfigure}[b]{0.3\textwidth}
        \includegraphics[width=\textwidth]{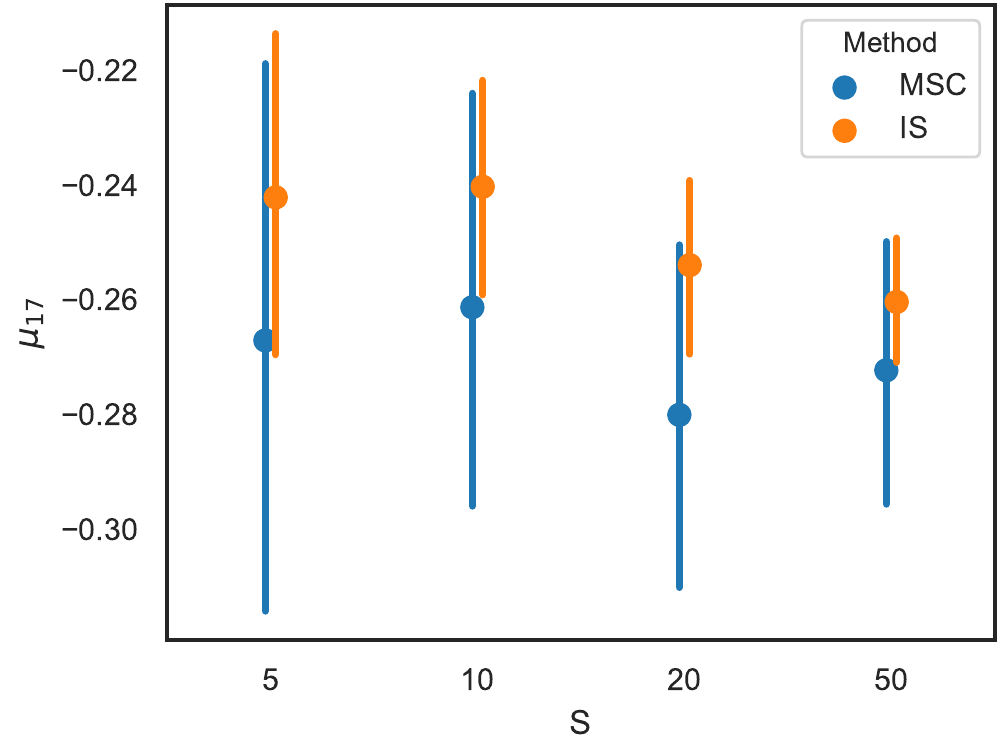}
        \caption{Ionos $\mu_{17}^\star$}
    \end{subfigure}
    ~ 
    \begin{subfigure}[b]{0.3\textwidth}
        \includegraphics[width=\textwidth]{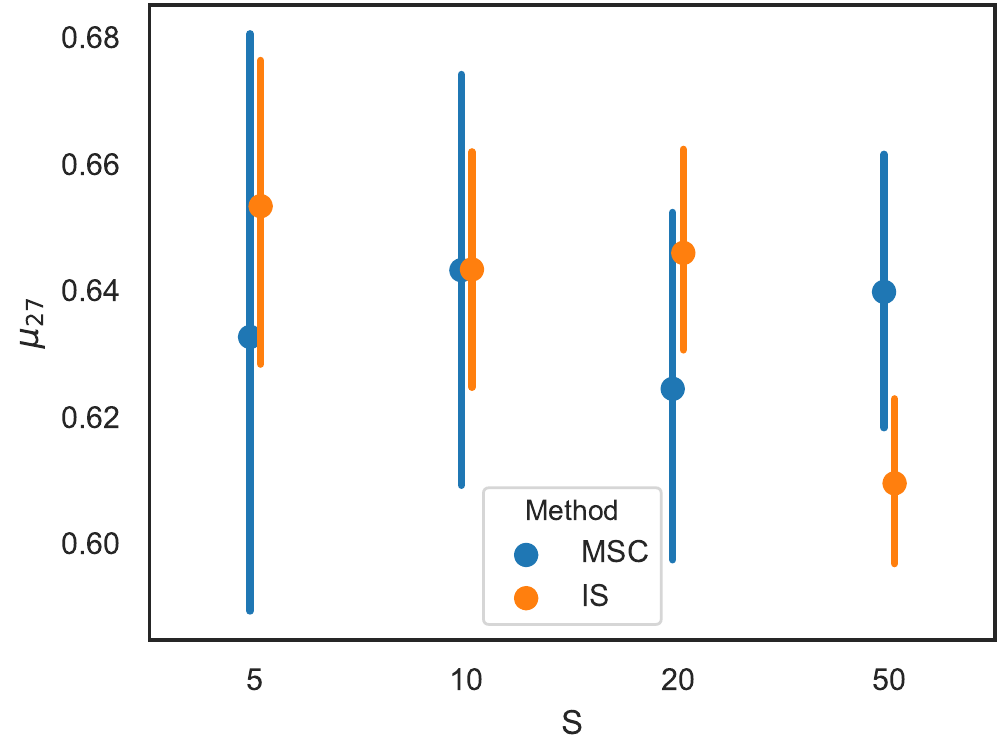}
        \caption{Ionos $\mu_{27}^\star$}
    \end{subfigure}
	\caption{\gls{MSC} is more robust to the number of samples $S$. The fitted mean parameter $\mu^\star$, for three representative dimensions of $\latent$, of \gls{MSC} (this paper) and \gls{IS} (\cf \citep{Bornschein2015}) on the Ionos and Heart datasets. The error bars corresponds to $100$ random initializations.}\label{fig:probit:biased}
	\vspace{-4mm}
\end{figure*}

\subsection{Bayesian Probit Regression}
Probit regression is commonly used for binary classification in machine learning and statistics. The Bayesian probit regression model assigns a Gaussian prior to the parameters. The prior and likelihood are
$p(\latent) = \Norm\left(\latent\idxby 0, I\right)$, $\Prb(\regdata_t = y\given \latent, \data_t) = \Phi(\data_t^\top \latent)^{y}\left(1-\Phi(\data_t^\top \latent)\right)^{1-y}$,
where $y \in \{0,1\}$ and $\Phi(\cdot)$ is the cumulative distribution function of the normal distribution. We apply the model for prediction in several UCI datasets \citep{Dua:2019}. We let the variational approximation be a Gaussian distribution $q(\latent\idxby\varparams) = \Norm\left(\latent\idxby \mu, \Sigma\right)$,
where $\Sigma $ is a diagonal covariance matrix. We compare \gls{MSC} (this paper) with the biased \gls{IS}-based approach  (\cf \cref{eq:sgd:inclkl} and \citep{Bornschein2015}) and \gls{EP} \citep{minka2001expectation} that minimizes the inclusive \gls{KL} locally. For \gls{SGD} methods we use adaptive step-sizes \citep{kingma2014adam}.

\Cref{tbl:expts:probit} illustrates the predictive performance of the fitted model on held-out test data. The results where generated by splitting each dataset $100$ times into $90\%$ training and $10\%$ test data, then computing average prediction error and its standard deviation. \gls{MSC} performs as well as \gls{EP} which is particularly well suited to this problem. However, \gls{EP} requires more model-specific derivations and can be difficult to implement when the moment matching subproblem can not be solved in closed form. In these experiments the bias introduced by \gls{IS} does not significantly impact the predictive performance compared to \gls{MSC}.

\defcitealias{Bornschein2015}{B.B. 2015}
\begin{table}[th]
	\centering
	\begin{tabular}{ccccc}
	\hline 
	Dataset & \acrshort{EP} \citep{minka2001expectation} & \acrshort{IS} \citep{Bornschein2015} & \acrshort{MSC} (adaptive) & \acrshort{MSC} (prior)  \\
	\hline
	Pima & $0.227 \pm 0.048$ & $0.229 \pm 0.047$ & $0.227\pm 0.046$ & $0.456\pm 0.093$ \\
	Ionos & $0.115 \pm 0.053$ & $0.115 \pm 0.054$  & $0.117 \pm 0.053$  & $0.182 \pm 0.070$ \\
	Heart & $0.161 \pm 0.066$ & $0.163 \pm 0.066$  & $0.160 \pm 0.063$ & $0.342 \pm 0.11$
	\end{tabular}
	\caption{Test error for Bayesian probit regression; lower is better. Estimated using  \acrshort{EP} \citep{minka2001expectation}, \acrshort{IS} (\cf \citep{Bornschein2015}), and \acrshort{MSC} (this paper) with proposal $q(\latent\idxby\varparams)$ (adaptive) or $p(\latent)$ (prior) for $3$ UCI datasets. Predictive performance is comparable, but \gls{MSC} is more robust and generically applicable.}\label{tbl:expts:probit}
	\vspace{-4mm}
\end{table}
We compare how the approximations based on \gls{MSC} and \gls{IS} are affected by the number of samples $S$ at each iteration. In \cref{fig:probit:biased} we plot the mean value $\mu^\star$ based on $100$ random initializations for several values of $S$ on the Heart and Ionos datasets. The \gls{MSC} is more robust to the choice of $S$, converging to similar mean values for all the choices of $S$ in this example. For the Heart dataset, \gls{IS} clearly struggles with a bias for low values of the number of samples $S$.

\subsection{Stochastic Volatility}
The stochastic volatility model is commonly used in financial econometrics \citep{Chib2009}.
The model is $p(\latent_0\idxby \modparams) = \Norm\left(\latent_0\idxby 0, \frac{\sigma^2}{1-\phi^2}\right), \quad 
	p(\latent_t\given\latent_{t-1}\idxby \modparams) = \Norm\left(\latent_t\idxby \mu + \phi (\latent_{t-1}-\mu), \sigma^2\right)$, $p(\data_t\given\latent_t\idxby \modparams) = \Norm\left(\data_t\idxby 0, \beta\exp(\latent_t)\right)$,
where the parameters are constrained as follows $\modparams = \left(\sigma^2, \phi, \mu, \beta\right)  \in \reals_+ \times (-1,1)\times \reals \times \reals_+$.  Both the posterior distribution and log-marginal likelihood are intractable so we make use of \cref{alg:msa:modparams} as outlined in \cref{sec:modparams} with the \gls{CSMC} kernel described in the supplement. The proposal distributions are  $q(\latent_0\idxby \modparams, \varparams_0) \propto p(\latent_0\idxby \modparams) e^{-\frac{1}{2}\Lambda_0\latent_0^2 + \nu_0 \latent_0 }$, $q(\latent_t\given\latent_{t-1}\idxby \modparams, \varparams_t) \propto p(\latent_t\given\latent_{t-1}\idxby \modparams) \,e^{-\frac{1}{2}\Lambda_t \latent_t^2 + \nu_t \latent_t }$,
with variational parameters $\varparams_t = (\nu_t, \Lambda_t) \in \reals \times \reals_+$. 
\begin{wrapfigure}{r}{0.45\textwidth}
\vspace{-3mm}
	\centering
	\includegraphics[width=0.35\columnwidth]{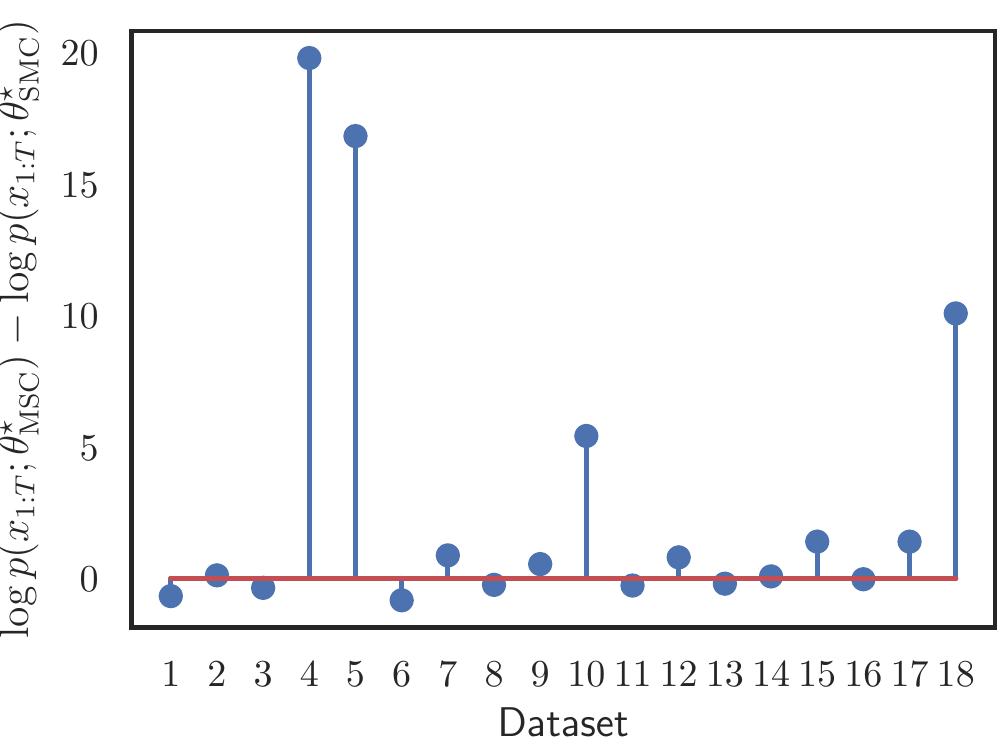}
	\caption{Difference in log-marginal likelihood values for parameters learnt by \gls{MSC} (this paper) and \gls{SMC} \citep{gu2015neural}. The likelihood obtained by \gls{MSC}, on average, is superior to or as good as that obtained by \gls{SMC}.}\label{fig:sv}
\vspace{-7mm}
\end{wrapfigure}
We compare \gls{MSC} with the \gls{SMC}-based approach \citep{gu2015neural} using adaptive step-size \citep{kingma2014adam}.
We study monthly returns over $10$ years ($9/2007$ to $8/2017$) for the exchange rate of $18$ currencies with respect to the US dollar. The data is obtained from the Federal Reserve System. In \cref{fig:sv} we illustrate the difference between the log-marginal likelihood obtained by the two methods, $\log p(\data\idxby\modparams_{\textrm{MSC}}^\star)-\log p(\data\idxby\modparams_{\textrm{SMC}}^\star)$. We learn the model and variational  parameters using $S=10$ particles for both methods, and estimate the log-marginal likelihood after convergence using $S=10,000$. The log-marginal likelihood obtained by \gls{MSC} is significantly better than \gls{SMC} for several of the datasets.

\section{Conclusions}\label{sec:conclusion}

In \gls{VI}, the properties of the approximation $q$, to the posterior $p$, depends on the choice of divergence that is minimized. The most common choice is the exclusive \gls{KL} divergence $\KL{q}{p}$, which is computationally convenient, but known to suffer from underestimation of the posterior uncertainty. An alternative, which has been our focus here, is the inclusive \gls{KL} divergence $\KL{p}{q}$.
The benefit of using the inclusive \gls{KL} is to obtain a more ``robust'' approximation that does not underestimate the uncertainty. However, in this paper we have argued, and illustrated numerically, that such underestimation of uncertainty can still be an issue, if the optimization is based on \emph{biased} gradient estimates,
as is the case for previously proposed \gls{VI} algorithms.
As a remedy, we introduced \acrlong{MSC}, a new way to reliably learn a variational approximation that minimizes the inclusive \gls{KL}. This results in a method that melds \gls{VI} and \gls{MCMC}. We have illustrated its convergence properties on a simple toy example, and studied its performance on
  Bayesian probit regression for classification 
  as well as a stochastic volatility model for financial data.

\section*{Broader Impact}
\Gls{MSC} is a general purpose approximate statistical inference method. The main goal is to remove systematic errors due to biased estimates of the gradient of the optimization objective function. This can allow for more reliable and robust inferences based on the posterior approximation. However, just like other standard inference methods it does not protect from any bias introduced by applying it to specific models and data \citep{corbett2018measure,mehrabi2019survey}.

\begin{ack}
This work is supported by ONR N00014-17-1-2131, ONR N00014-15-1-2209, NIH 1U01MH115727-01, NSF CCF-1740833, DARPA SD2 FA8750-18-C-0130, Amazon, NVIDIA, and the Simons Foundation.
Fredrik Lindsten is financially supported by the Swedish Research Council (project 2016-04278), by the Swedish Foundation for Strategic Research (project ICA16-0015) and by the Wallenberg AI, Autonomous Systems and Software Program (WASP) funded by the Knut and Alice Wallenberg Foundation.
\end{ack}


\bibliographystyle{abbrvnat}
\bibliography{references}

\newpage
\appendix
\section{Supplementary Material}\label{sec:supp}
\subsection{Subset Average Likelihood}\label{app:like}
 \citet{li2016renyi,dieng2017chi}, who study different classes of divergences where the likelihood also enters as a product, propose to replace the true likelihood at each iteration with a ``subset average likelihood''. The subset average likelihood approach makes the following approximation
\begin{align*}
	p(\data\given\latent) = \prod_{i=1}^n p(x_i\given\latent) \approx p(\data_M\given\latent)^{\frac{n}{m}} \eqdef  \prod_{j \in M} p(x_{j}\given\latent)^{\frac{n}{m}},
\end{align*}
where $M \subset \{1,\ldots,n\}$ is a set of indices corresponding to a mini-batch of size $m=|M|$ data points sampled uniformly from $\{x_1, \ldots,x_n\}$ with or without replacement. Considering the same approach for the inclusive \gls{KL} case the \emph{unbiased} stochastic gradient obtained is
\begin{align}
	\frac{1}{S}\sum_{s=1}^S \frac{p(\latent^s) p(\data_M\given\latent^s)^{\frac{n}{m}}}{q(\latent^s\idxby \varparams)} \grad_\varparams \log q(\latent^s\idxby \varparams), ~ \latent^s \sim q(\latent\idxby \varparams).
	\label{eq:avglike:sgd}
\end{align}
This approximation also leads to a systematic error in the \gls{SGD} algorithm. It is no longer minimizing the \gls{KL} divergence from the posterior $p$ to the variational approximation $q$. In fact, it is possible to show that it is actually minimizing the \gls{KL} divergence from a \emph{perturbed posterior} $\widetilde{p}$, where the likelihood is replaced by a mixture of all potential subset average likelihoods, to the variational approximation $q$. This result is formalized by \cref{prop:avglike}.
\begin{prop}
Using the stochastic gradient defined by \cref{eq:avglike:sgd} and an iterative \gls{SGD} algorithm according to \cref{eq:sa:algorithm} the fixed points $\varparams^\star$ are identical to the solution to
\begin{align}
	\nabla_\varparams \mathrm{KL}\left(\widetilde{p}(\latent\given\data) \|  q(\latent\idxby \varparams) \right) = 0,
\end{align}
where the perturbed posterior $\widetilde{p}$, if it exists, is given by
\begin{align*}
	\widetilde{p}(\latent\given\data) &\propto p(\latent) \sum_{M \in \mathcal{M}} p(\data_M\given\latent)^{\frac{n}{m}},
\end{align*}
and $\mathcal{M}$ is the set of all possible combinations of mini-batches $M$ of size $m$.
\label{prop:avglike}
\end{prop}
\begin{proof}
See the Supplementary Material.
\end{proof}
In the supplement we provide illustrations on a simulated example. It is in general difficult to determine the magnitude of the error introduced by the subset average likelihood in practical applications. The subset average likelihood approach for R\'{e}nyi and $\chi^2$ divergences \citep{li2016renyi,dieng2017chi} likewise leads to a systematic error in the stochastic gradient. Furthermore, the fixed points of the resulting stochastic systems for these divergences are difficult to quantify, making it even harder to understand the effect of the approximation.

\subsection*{Conditional Sequential Monte Carlo}\label{sec:csmc}
Just like \gls{CIS} is a straightforward modification of \gls{IS}, so is \gls{CSMC} a straightforward modification of \gls{SMC}.
We make use of \gls{CSMC} with ancestor sampling as proposed by \citet{lindsten2014particle} combined with twisted \gls{SMC} \citep{guarniero2017iterated,heng2017controlled,naesseth2019elements}. While \gls{SMC} can be adapted to perform inference for almost any probabilistic model \citep{naesseth2019elements}, we here focus on the state space model
\begin{align*}
	p(\latent_{1:T},\data_{1:T}) &= p(\latent_1) p(\data_1\given\latent_1) \prod_{t=1}^T p(\latent_t\given \latent_{t-1}) p(\data_t\given\latent_t),
\end{align*}
where we assume that the prior $p(\latent_1)$ and transition $p(\latent_t\given \latent_{t-1})$ are conditionally Gaussian. Because the prior and transition distributions are Gaussian it is convenient to define the full approximation to the posterior $p(\latent_{1:T}\given\data_{1:T})$ to be the multivariate normal
\begin{align}
	q(\latent_{1:T} \idxby\varparams) &= q(\latent_{1} \idxby\varparams_1) \prod_{t=2}^T q(\latent_{t}\given \latent_{t-1} \idxby\varparams_t), \\
	q(\latent_{1} \idxby\varparams_1) &\propto p(\latent_1) \psi(\latent_{1} \idxby\varparams_1) \nonumber, \\
	q(\latent_{t}\given \latent_{t-1} \idxby\varparams_t) &\propto p(\latent_t\given \latent_{t-1}) \psi(\latent_{t}  \idxby\varparams_t) \nonumber,
\end{align}
where $\psi$ are \emph{twisting potentials}
\begin{align*}
	\psi(\latent_{t}  \idxby\varparams_t) &= \exp\left(-\frac{1}{2}\latent_t^\top \Lambda_t \latent_t + \nu_t^\top \latent_t \right),
\end{align*}
with $\varparams_t = (\Lambda_t,\nu_t)$. We are now equipped to explain the \gls{CSMC} kernel that updates a conditional trajectory $\latent_{1:T}[k-1] = (\latent_1[k-1],\ldots \latent_T[k-1])$. Each iteration of \gls{CSMC} consists of three steps: initialization for $t=1$, running a modified \gls{SMC} algorithm for $t>1$, and then updating the conditional sample for the next iteration. We explain in detail below.

First, perform (conditional) \gls{IS} for the first step where $t=1$. Set $\latent_1^1 = \latent_1[k-1]$ and propose the remaining $S-1$ samples from a proposal distribution $q$
\begin{align*}
	\latent_1^i &\sim q(\latent_1\idxby\varparams_1), \quad i=2,\ldots,S
\end{align*}
and compute the importance weights for $i=1,\ldots,S$
\begin{align*}
	w_1^i &= \frac{p(\latent_1^i) \psi(\latent_{1}^i  \idxby\varparams_1)}{q(\latent_1^i\idxby\varparams_1)}, &\bar{w}_1^i = \frac{w_1^i}{\sum_{j=1}^S w_1^j}.
\end{align*}

Then, for each step $t>1$ in turn perform \emph{resampling}, \emph{ancestor sampling}, \emph{propagation} and \emph{weighting}. Resampling picks the most promising earlier sample to propagate, \ie for $i=2,\ldots,S$ simulate \emph{ancestor} variables $a_{t-1}^i$ with probability
\begin{align*}
	\Prb\left(a_{t-1}^i =j \right) &= \bar{w}_{t-1}^j.
\end{align*}
For $i=1$ instead, simulate the corresponding ancestor variable $a_{t-1}^1$ with probability
\begin{align*}
	\Prb\left(a_{t-1}^1 =j \right) &\propto \bar{w}_{t-1}^j q(\latent_t[k-1]\given \latent_{t-1}^j\idxby \varparams_t),
\end{align*}
where $\latent_t[k-1]$ is the corresponding element of the conditional trajectory $\latent_{1:T}$ from the previous iteration. This is known as ancestor sampling \citep{lindsten2014particle}.

When propagating for $i=1$ simply set $\latent_t^1=\latent_t[k-1]$, and simulate the remainder from the proposal distribution
\begin{align*}
	\latent_t^i &\sim q(\latent_t\given\latent_{t-1}^{a_{t-1}^i}\idxby\varparams), \quad i=2,\ldots,S
\end{align*}
Set $\latent_{1:t}^i = (\latent_{1:t-1}^{a_{t-1}^i}, \latent_t^i)$ and compute the weights for all $i=1,\ldots,S$
\begin{align}
	w_t^i &= \frac{p(\latent_t^i\given \latent_{t-1}^{a_{t-1}^i})p(\data_{t-1}\given\latent_{t-1}^{a_{t-1}^i})}{q(\latent_t^i\given \latent_{t-1}^{a_{t-1}^i}\idxby\varparams_t)} \frac{\psi(\latent_{t}^i  \idxby\varparams_t)}{\psi(\latent_{t-1}^{a_{t-1}^i}  \idxby\varparams_{t-1})}, \label{eq:weights:smc} \\
	\bar{w}_t^i &= \frac{w_t^i}{\sum_{j=1}^S w_t^j}. \nonumber
\end{align}
For the final step $t=T$ the (unnormalized) weights are instead
\begin{align}
	w_T^i &= \frac{p(\latent_T^i\given \latent_{T-1}^{a_{T-1}^i})p(\data_{T-1}\given\latent_{T-1}^{a_{T-1}^i})}{q(\latent_T^i\given \latent_{T-1}^{a_{T-1}^i}\idxby\varparams_T)} \frac{p(\data_{T}\given\latent_{T}^{i})}{\psi(\latent_{t-1}^{a_{t-1}^i}  \idxby\varparams_{t-1})}. 
	\label{eq:weights:smc:last}
\end{align}

Finally, an updated conditional sample is generated by picking one of the proposed trajectories with probability proportional to its (normalized) weight, \ie
\begin{align*}
	\latent_{1:T}[k] &= \latent_{1:T}^J, \quad 
\end{align*}
where $J$ is a discrete random variable with probability $\Prb(J=j) = \bar{w}_T^j$.

Repeating this procedure iteratively constructs a Markov chain with the posterior $p(\latent_{1:T}\given\data_{1:T})$ as its stationary distribution \citep{andrieu2010particle,lindsten2014particle,naesseth2019elements}. With this it is possible to attain an estimate of the gradient with respect to the variational parameters of \cref{eq:kloptproblem} as follows
\begin{align}
	\widehat{g}_{\textrm{KL}}(\varparams) &= -\grad_{\varparams} \log q(\latent_{1:T}[k]\idxby\varparams),
	\label{eq:kloptproblem:cisgradient}
\end{align}
where $\latent_{1:T}[k]$ is the conditional sample retained at iteration $k$ of the \gls{CSMC} algorithm. 

We summarize one full iteration of the \gls{CSMC} algorithm in \cref{alg:csmc}. This algorithm defines a Markov kernel $M(\latent_{1:T}[k]\given \latent_{1:T}[k-1]\idxby\varparams)$ useful for \gls{MSC}.
\begin{algorithm}
\DontPrintSemicolon
\SetKwInOut{Input}{Input}\SetKwInOut{Output}{Output}
 \Input{Model $p(\latent_{1:T},\data_{1:T})$, proposal $q(\latent_{1:T}\idxby\varparams)$, conditional sample $\latent_{1:T}[k-1]$, and total number of internal samples $S$.}
 \Output{$\latent_{1:T}[k] \sim M(\cdot\given \latent_{1:T}[k-1]\idxby\varparams)$, updated conditional sample.}
 \BlankLine
 Set $\displaystyle \latent_1^1 = \latent_1[k]$\;
 
 Sample $\displaystyle \latent_1^i \sim q(\latent_1\idxby\varparams_1)$ for $\displaystyle i=2,\ldots,S$\;
 
 Compute $\displaystyle w_1^i = \frac{p(\latent_1^i)\psi(\latent_{1}^i  \idxby\varparams_1)}{q(\latent_1^i\idxby\varparams_1)}$ for $\displaystyle i=1,\ldots,S$ \;
 
 \For{$\displaystyle t=2,\ldots,T$}{
 	\For{$\displaystyle i=2,\ldots,S$}{
    	Sample $a_{t-1}^i$ \wip $\Prb(a_{t-1}^i=j) = \bar{w}_{t-1}^j$\;
    	
    	Sample $\latent_t^i \sim q(\latent_t\given\latent_{t-1}^{a_{t-1}^i}\idxby \varparams_t)$\;
    }
    Sample $a_{t-1}^1$ \wip \\$\Prb(a_{t-1}^i=j) \propto \bar{w}_{t-1}^j q(\latent_t[k]\given\latent_{t-1}^{j}\idxby \varparams_t)$\;
    
    Set $\latent_t^1 = \latent_t[k]$\;
    
    \For{$\displaystyle i=1,\ldots,S$}{
    	Compute $w_t^i$ in \cref{eq:weights:smc} or \cref{eq:weights:smc:last}\;
    	
    	Set $\latent_{1:t}^i = \left(\latent_{1:t-1}^{a_{t-1}^i}, \latent_t^i\right)$
    }
} 
 
Sample $\displaystyle J$ with probability $\displaystyle \Prb(J=j) \propto  \bar{w}_T^j$\;
 
 Set $\displaystyle \latent_{1:T}[k] = \latent_{1:T}^J$\;
 \caption{Conditional Sequential Monte Carlo}\label{alg:csmc}
\end{algorithm}

\subsection*{Proof of Proposition 1}
This result is an adaptation of \citet[Theorem 1]{Gu7270} based on \citet[Theorem 3.17, page 304]{benveniste2012adaptive}. 
Let $\varparams^\star$ be a minimizer of the inclusive \gls{KL} divergence in \cref{eq:kloptproblem}. Consider the \gls{ODE} defined by
\begin{align}
	\frac{\dif}{\dif t}\varparams(t) &= \Exp_{p(\latent\given\data)}\left[-s(\latent\idxby\varparams(t))\right], \quad \varparams(0) = \varparams_0,
	\label{eq:ode}
\end{align}
and its solution $\varparams(t)$, $t\geq 0$. If the \gls{ODE} in \cref{eq:ode} admits the unique solution $\varparams(t) = \widehat{\varparams}$, $t \geq 0$ for $\varparams(0)=\widehat{\varparams}$, then $\widehat{\varparams}$ is called a stability point. The minimzer $\varparams^\star$ is a stability point of \cref{eq:ode}. A set $\Lambda$ is called the domain of attraction of $\widehat{\varparams}$, if the solution to \cref{eq:ode} for $\varparams(0)\in \Lambda$ remains in $\Lambda$ and converges to $\widehat{\varparams}$. 
Suppose that $\varparams_k \in \reals^d$ and that $\Lambda$ is an open set in $\reals^{\dimvp}$. Furthermore, suppose $\latent[k]\in \reals^{\dimlv}$ and that $Z$ is an open set in $\reals^{\dimlv}$. Denote the Markov kernel in \gls{MSC}, \cref{alg:msa}, by $M_\varparams(\latent, \dif\latent')$ and repeated application of it by $M_\varparams^k(\latent, \dif\latent') = \int \cdots \int M_\varparams(\latent, \dif\latent_1) M_\varparams(\latent_2, \dif\latent_3) \cdots M_\varparams(\latent_{k-1}, \dif\latent')$. $|\latent|$ denotes the length of the vector $\latent$. Let $Q$ be any compact subset of $\Lambda$, and $q>1$ a sufficiently large real number such that the following assumptions hold. We follow \citet{Gu7270} and assume:
\begin{cond}
Assume that the step size sequence satisfies $\sum_{k=1}^\infty \step_k = \infty$ and $\sum_{k=1}^\infty \step_k^2 < \infty$.
\label{cond:1}
\end{cond}
\begin{cond}[Integrability]
There exists a constant $C_1$ such that for any $\varparams\in\Lambda$, $\latent \in Z$ and $k\geq 1$,
\begin{align*}
	\int \left(1+ |\latent'|^q\right)M_\varparams^k(\latent, \dif\latent') \leq C_1 \left(1+ |\latent|^q\right)
\end{align*}
\label{cond:2}
\end{cond}
\begin{cond}[Convergence of the Markov Chain]
Let $p(\latent\given\data)$ be the unique invariant measure for $M_\varparams$. For each $\varparams \in \Lambda$,
\begin{align*}
	\lim_{k\to\infty} \sup_{\latent\in Z} \frac{1}{1+|\latent|^q} \int  \left(1+ |\latent'|^q\right) |M_\varparams^k(\latent, \dif\latent') - p(\dif\latent'\given\data)| = 0.
\end{align*}
\label{cond:3}
\end{cond}
\begin{cond}[Continuity in $\varparams$]
There exists a constant $C_2$, such that for all $\varparams, \varparams' \in Q$
\begin{align*}
	\left| \int \left(1+|\latent'|^q\right) \left(M_\varparams(\latent, \dif\latent')-M_{\varparams'}(\latent, \dif\latent')\right)\right| \leq C_2 |\varparams-\varparams'|\left(1+|\latent|^q\right).
\end{align*}
\label{cond:4}
\end{cond}
\begin{cond}[Continuity in $\latent$]
There exists a constant $C_3$, such that for all $\latent_1, \latent_2 \in Z$
\begin{align*}
	\sup_{\varparams \in \Lambda} \left| \int \left(1+|\latent'|^{q+1}\right) \left(M_\varparams(\latent_1, \dif\latent')-M_{\varparams}(\latent_2, \dif\latent')\right)\right| \leq C_3 |\latent_1-\latent_2|\left(1+|\latent_1|^q+|\latent_2|^q\right).
\end{align*}
\label{cond:5}
\end{cond}
\begin{cond}[Conditions on the Score Function]
For any compact subset $Q \subset \Lambda$, there exist positive constants $p$, $K_1$, $K_2$, $K_3$ and $\nu > \nicefrac{1}{2}$ such that for all $\varparams, \varparams'\in\Lambda$ and $\latent, \latent_1, \latent_2 \in Z$,
\begin{align*}
	|\grad_\varparams \log q(\latent\idxby\varparams)| &\leq K_1 \left(1+|\latent|^{p+1}\right), \\
	|\grad_\varparams \log q(\latent_1\idxby\varparams) - \grad_\varparams \log q(\latent_2\idxby\varparams)| &\leq K_2 |\latent_1-\latent_2|\left(1+|\latent_1|^{p}+|\latent_2|^{p}\right), \\
	|\grad_\varparams \log q(\latent\idxby\varparams) - \grad_\varparams \log q(\latent\idxby\varparams')| &\leq K_3 |\varparams-\varparams'|^\nu\left(1+|\latent|^{p+1}\right). 
\end{align*}
\label{cond:6}
\end{cond}
The constants $C_1, \ldots, C_3$ and $\nu$ may depend on the compact set $Q$ and the real number $q$. \Cref{cond:1} is the standard Robbins-Monro condition and \Cref{cond:6} controls the regularity of the model.  \Cref{cond:2}-\Cref{cond:5}  have to do with the convergence and continuity of the Markov kernel. These conditions can be difficult to verify in the general case, but can be proven more easily under the simplifying assumption that Z is compact. See \citet[Appendix B]{lindholm2019learning} for a proof of continuity of the \gls{CSMC} kernel, which can also be adapted to the \gls{CIS} kernel.

With the above assumptions the result follows from \citet[Theorem 1]{Gu7270} where (left - their notation, right - our notation)
\begin{align*}
	\theta &= \varparams,\\
	x &= \latent, \\
	\Pi_\theta &= M_\varparams, \\
	H(\theta,x) &=\grad_\varparams \log q(\latent\idxby\varparams),
\end{align*}
and $I(\theta,x) =0$,  $\Gamma_k = 0$.

\subsection*{Proof of Proposition 2}
The fixed points of the iterative algorithm are the solutions to the equation when we set the expectation of \cref{eq:avglike:sgd} equal to zero. The equation is given by
	\begin{align*}
		&\Exp\left[-\frac{1}{S}\sum_{s=1}^S \frac{p(\latent^s) p(\data_M\given\latent^s)^{\frac{n}{m}}}{q(\latent^s\idxby \varparams)} \grad_\varparams \log q(\latent^s\idxby \varparams)\right] 
		=
		\Exp\left[- \frac{p(\latent) p(\data_M\given\latent)^{\frac{n}{m}}}{q(\latent\idxby \varparams)} \grad_\varparams \log q(\latent\idxby \varparams)\right] \\
		&= \Exp\left[- \frac{p(\latent) \frac{1}{|\mathcal{M}|} \sum_{M\in\mathcal{M}}p(\data_M\given\latent)^{\frac{n}{m}}}{q(\latent\idxby \varparams)} \grad_\varparams \log q(\latent\idxby \varparams)\right] = 0 \\
		&\Longleftrightarrow \Exp_{\widetilde{p}(\latent\given\data) }\left[-\grad_\varparams \log q(\latent\idxby \varparams)\right]  = 0 \\
		 &\Longleftrightarrow \nabla_\varparams \mathrm{KL}\left(\widetilde{p}(\latent\given\data) \|  q(\latent\idxby \varparams) \right) = 0,
	\end{align*}
	where the first equality follows because the samples $\latent^s$ are independent and identically distributed. The second equality follows by the distribution of the mini-batches. The first equivalence follows because $\latent \sim q(\latent \idxby\varparams)$ and we multiply both sides by a constant independent of $\varparams$. The final equivalence follows because $\widetilde{p}(\latent\given\data)$ does not depend on $\varparams$. This concludes the proof.
	
\subsection*{Additional Results Bayesian Probit Regression}
We also compare the posterior uncertainty learnt using \gls{MSC} and \gls{IS}. \Cref{fig:probit} shows difference in the log-standard deviation between the posterior approximation learnt using \gls{MSC} and that using \gls{IS}, \ie $\log \sigma_{\mathrm{MSC}}^\star - \log \sigma_{\mathrm{IS}}^\star$. The figure contains one boxplot for each dimension of the latent variable and is based on data from $100$ random train-test splits. We can see that for two of the datasets, Heart and Ionos, \gls{MSC} on average learns a posterior approximation with higher uncertainty. However, for the Pima dataset the \gls{IS}-based method tends to learn higher variance approximations.
\begin{figure*}[t]
	\centering
	\begin{subfigure}[b]{0.32\textwidth}
        \includegraphics[width=\textwidth]{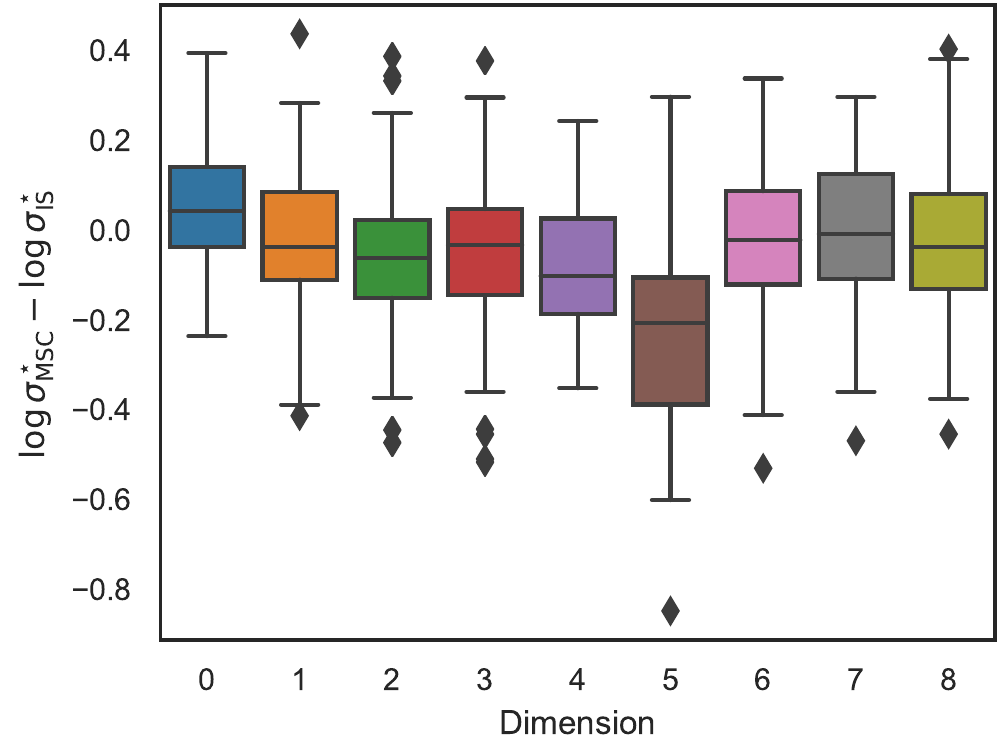}
        \caption{Pima}
    \end{subfigure}
    ~ 
    \begin{subfigure}[b]{0.32\textwidth}
        \includegraphics[width=\textwidth]{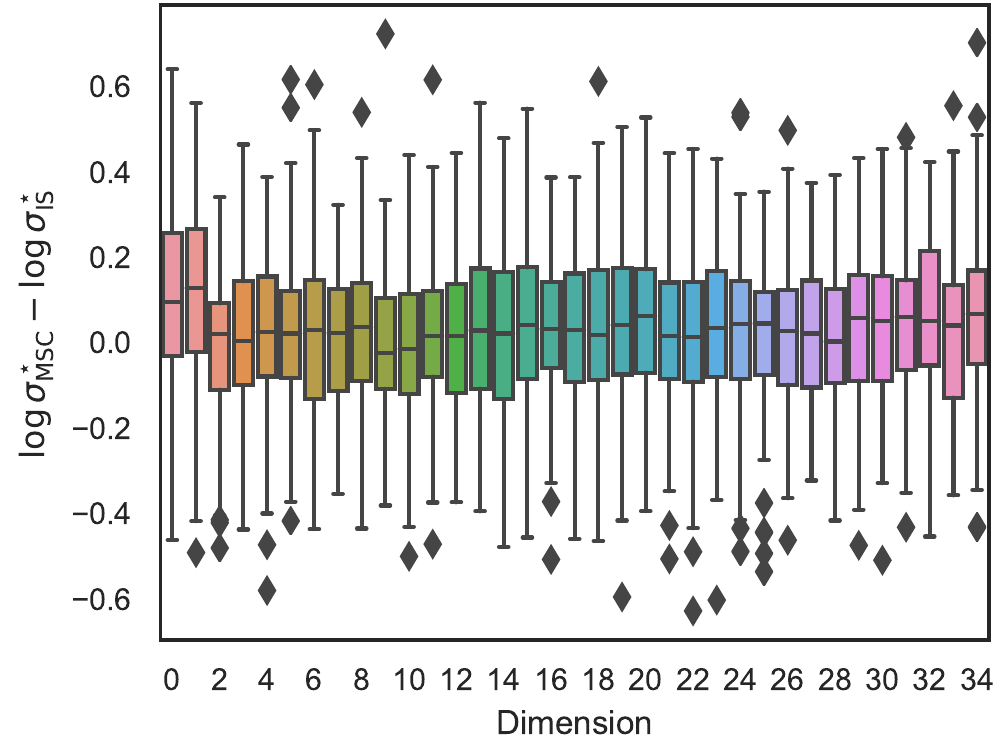}
        \caption{Ionos}
    \end{subfigure}
    ~ 
    \begin{subfigure}[b]{0.32\textwidth}
        \includegraphics[width=\textwidth]{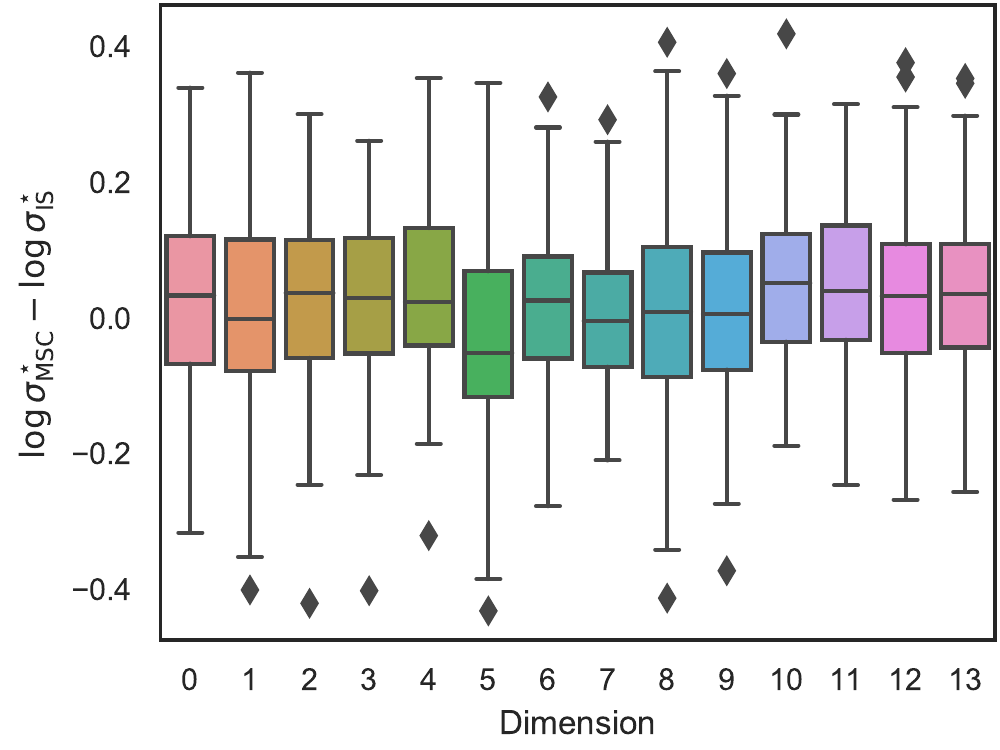}
        \caption{Heart}
    \end{subfigure}
	\caption{The difference in log-standard deviation of the variational approximation, $\log \sigma_{\mathrm{MSC}}^\star - \log \sigma_{\mathrm{IS}}^\star$, between parameters learnt using \gls{MSC} (this paper) and \gls{IS} (\cf \citep{Bornschein2015}). The dimension of the latent variable is plotted versus the parameters learnt from $100$ random splits.}\label{fig:probit}
\end{figure*}

\subsection*{Additional Results Subset Average Likelihoods}
We illustrate the difference between the true and perturbed posteriors in \cref{fig:supp:avg_likelihood} for a toy example where the two distributions can be computed exactly. The model is an unknown mean measured in Gaussian noise with a conjugate prior, \ie $\latent \sim \Norm(0,1)$, $x_i\sim \Norm(\latent, 1)$. To be able to exactly compute the perturbed posterior we keep the number of data points small $n=10$. The figure shows the true and perturbed posteriors for two randomly generated datasets with $m=2,5, 9$. 
\begin{figure}[h]
	\centering
	\begin{subfigure}[b]{0.48\textwidth}
        \includegraphics[width=\textwidth]{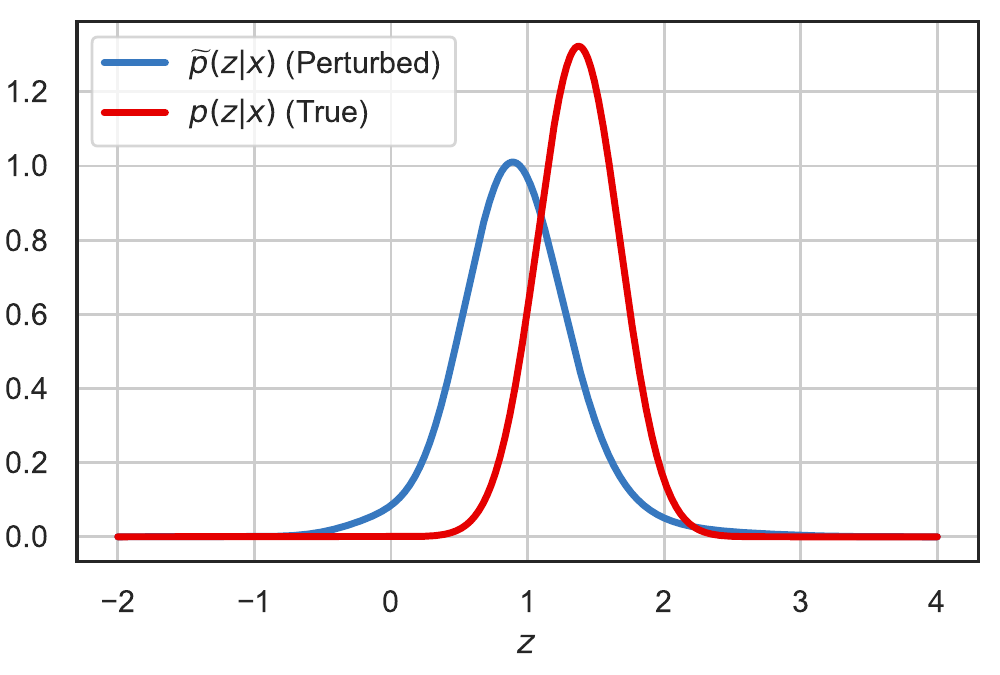}
    \end{subfigure}
    ~ 
    \begin{subfigure}[b]{0.48\textwidth}
        \includegraphics[width=\textwidth]{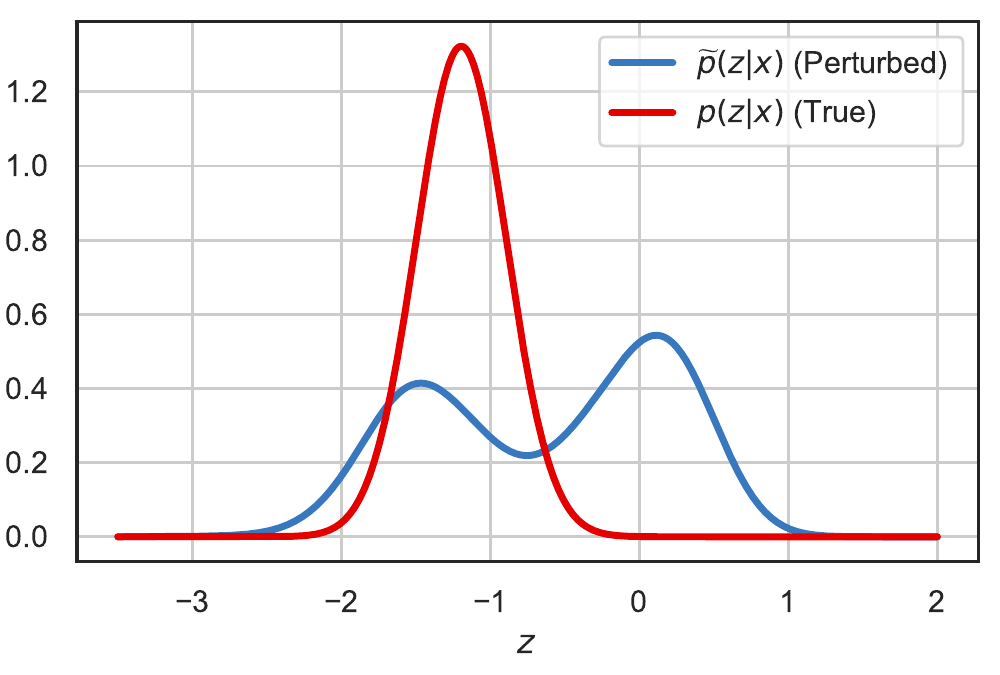}
    \end{subfigure}
    
	\begin{subfigure}[b]{0.48\textwidth}
        \includegraphics[width=\textwidth]{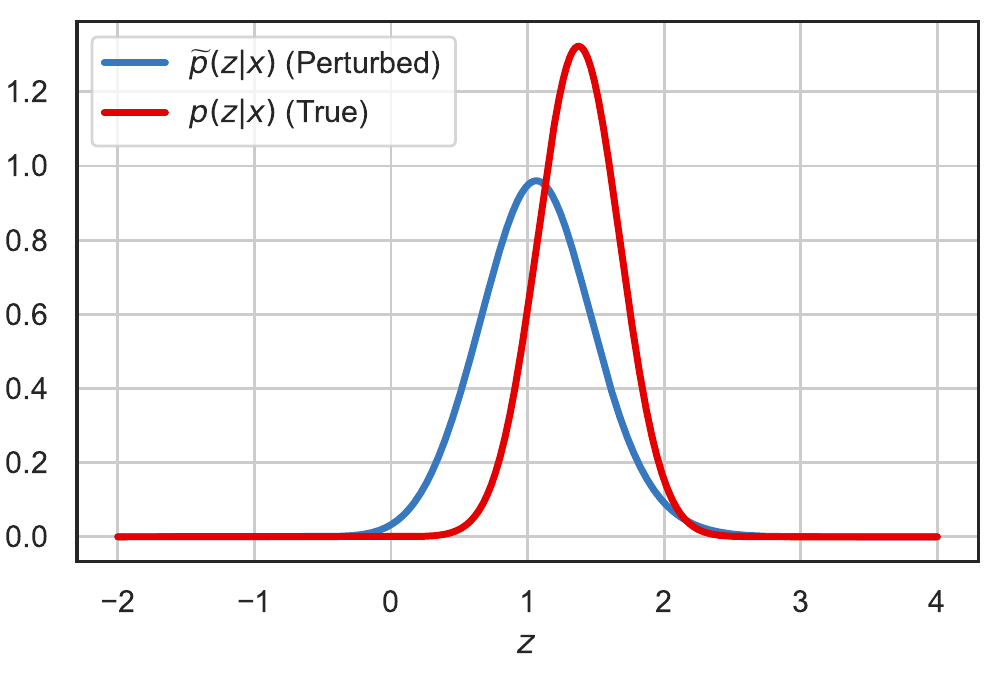}
    \end{subfigure}
    ~ 
    \begin{subfigure}[b]{0.48\textwidth}
        \includegraphics[width=\textwidth]{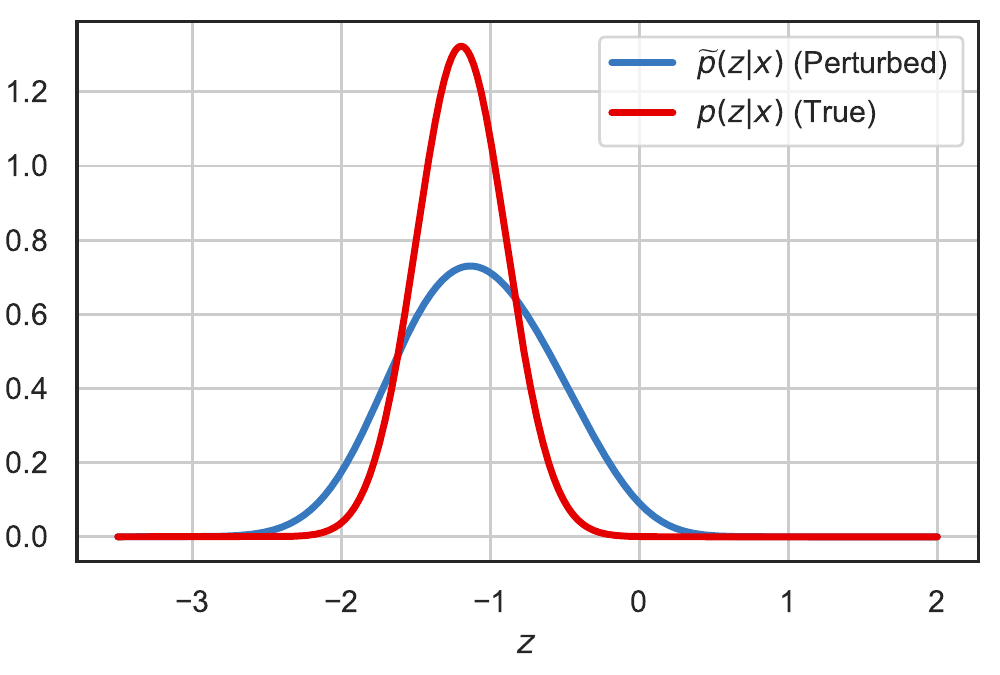}
    \end{subfigure}
    
    \begin{subfigure}[b]{0.48\textwidth}
        \includegraphics[width=\textwidth]{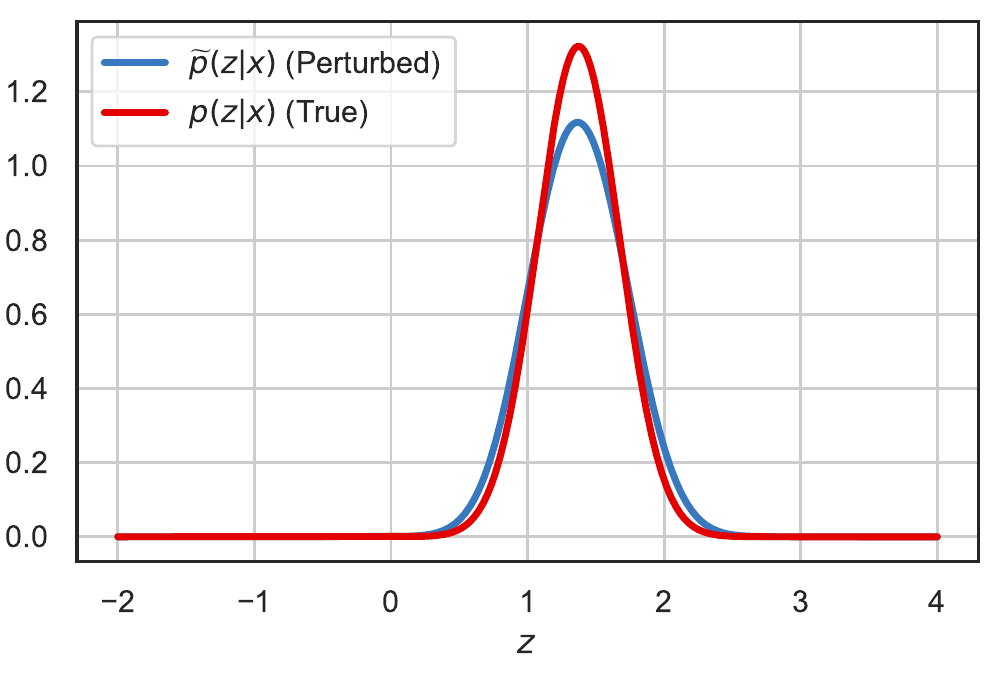}
    \end{subfigure}
    ~ 
    \begin{subfigure}[b]{0.48\textwidth}
        \includegraphics[width=\textwidth]{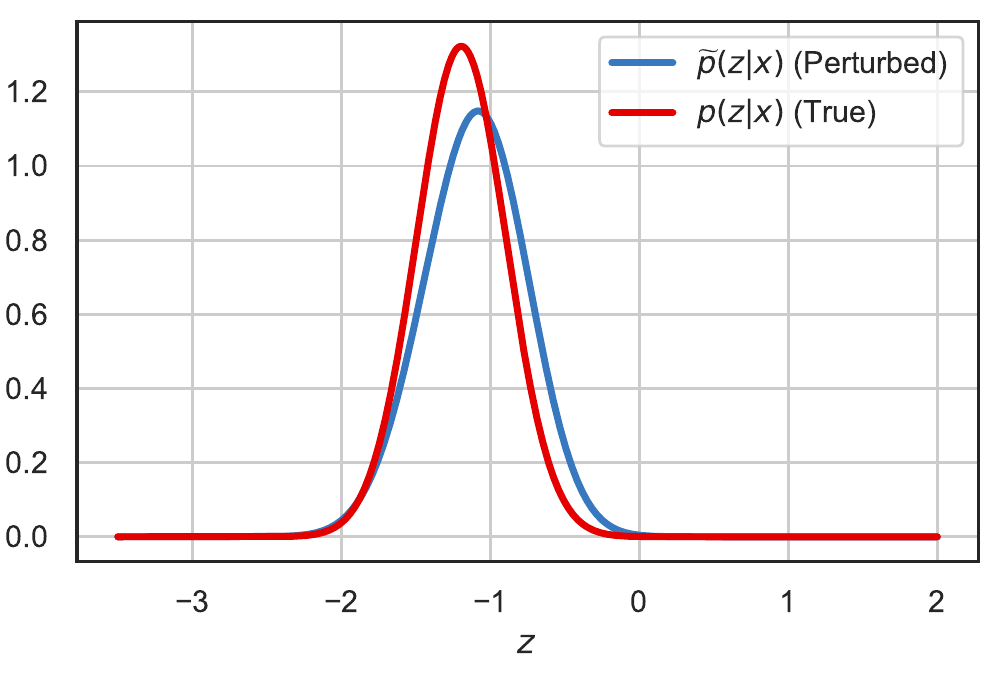}
    \end{subfigure}
	\caption{Example of perturbed and true posterior when using subset average likelihoods. The data used is simulated from the model defined by $\latent \sim \Norm(0,1)$, $x_i\sim \Norm(\latent, 1)$, $n=10$ for two different random seeds. The subset sizes where chosen to be $m=2$ (top row),  $m=5$ (top row) and $m=9$ (bottom row).}\label{fig:supp:avg_likelihood}
\end{figure}

\subsection*{Bias in \gls{CHIVI}}
\Cref{fig:supp:biased} illustrates the systematic error introduced in the optimal parameters of \gls{CHIVI} when using biased gradients.
\begin{figure}[h]
	\centering
	\begin{subfigure}[b]{0.4\textwidth}
        \includegraphics[width=\textwidth]{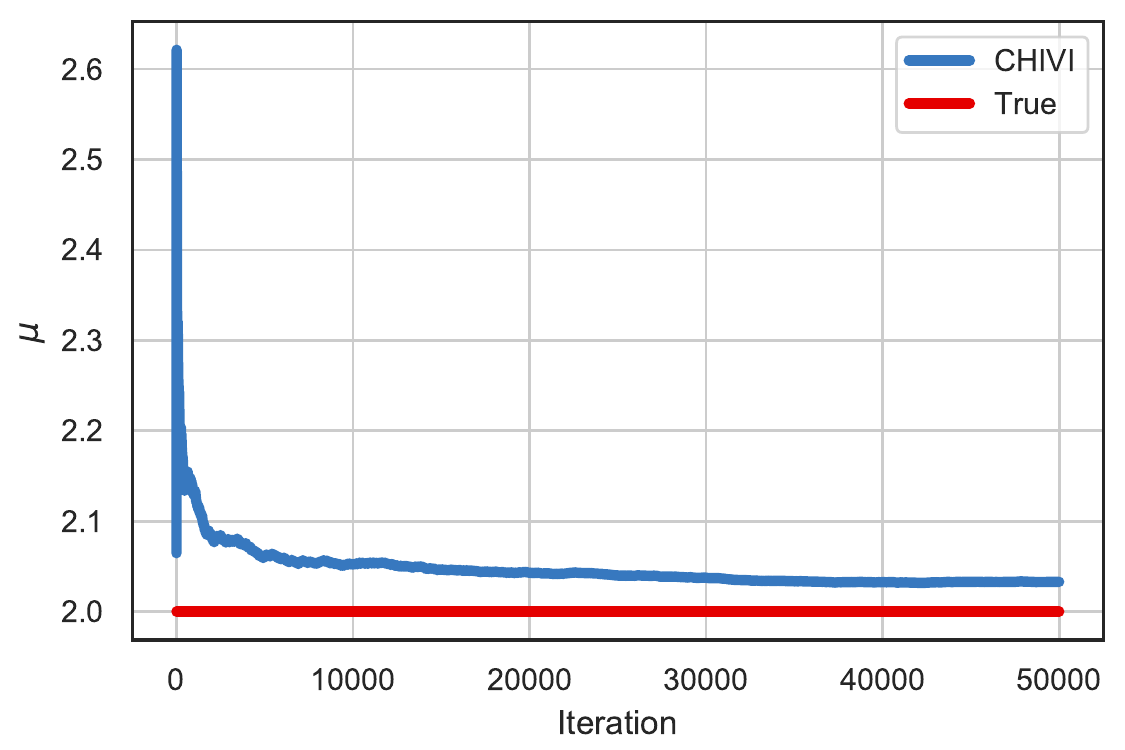}
    \end{subfigure}
    ~ 
    \begin{subfigure}[b]{0.4\textwidth}
        \includegraphics[width=\textwidth]{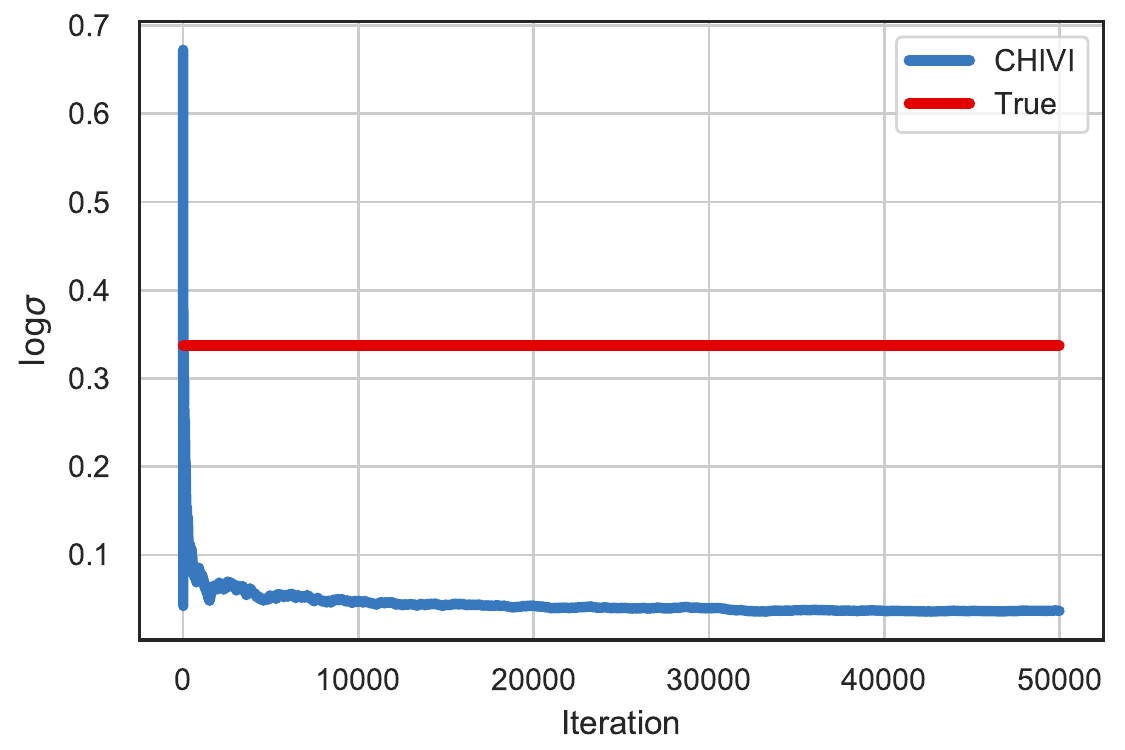}
    \end{subfigure}
	\caption{Example of learnt variational parameters for \gls{CHIVI}, as well as true parameters when using a Gaussian approximation to a skew normal posterior distribution.}\label{fig:supp:biased}
\end{figure}

\end{document}